\documentclass{article}

\usepackage{microtype}
\usepackage{graphicx}
\usepackage{subfigure}
\usepackage{booktabs} 

\usepackage{hyperref}

\usepackage{amsmath}
\usepackage{amssymb}
\usepackage{mathtools}
\usepackage{amsthm}

\usepackage[capitalize,noabbrev]{cleveref}

\usepackage[textsize=tiny]{todonotes}

\newcommand{\A}{\mathcal{A}}

\newcommand{\D}{\mathcal{D}}
\newcommand{\cD}{\D}
\newcommand{\cK}{\K}
\newcommand{\M}{\mathcal{M}}
\newcommand{\N}{\mathbb{N}}
\newcommand{\K}{\mathcal{K}}

\newcommand{\sign}{{\rm sign}}
\newcommand{\rank}{{\rm rank}}
\newcommand{\Trace}{{\rm Trace}}
\newcommand{\Gen}{{\rm Gen}}
\newcommand{\bzero}{\mathbf{0}}

\newcommand{\zl}[1]{\textsf{\color{red} Zhou : #1}}
\newcommand{\daogao}[1]{\textsf{\color{blue} Liu: #1}}

\newcommand{\E}{\mathbb{E}}

\newcommand{\R}{\mathbb{R}}

\newcommand{\Var}{\mathrm{Var}}

\newcommand{\Cov}{\mathrm{Cov}}
\renewcommand{\d}{\mathrm{d}}
\newcommand{\ox}{\overline{x}}
\DeclareMathOperator{\tr}{tr}

\newcommand{\cN}{\mathcal N}

\usepackage{thmtools}
\usepackage{thm-restate}

\usepackage[preprint]{neurips_2024}

\usepackage[utf8]{inputenc} 
\usepackage[T1]{fontenc}    
\usepackage{hyperref}       
\usepackage{url}            
\usepackage{booktabs}       
\usepackage{amsfonts}       
\usepackage{nicefrac}       
\usepackage{microtype}      
\usepackage{xcolor}         
\usepackage{graphicx}
\usepackage{amsmath,amssymb,amsfonts,amsthm,mathrsfs}
\usepackage{algorithm}
\usepackage{algorithmic}

\theoremstyle{plain}
\newtheorem{theorem}{Theorem}[section]
\newtheorem{proposition}[theorem]{Proposition}
\newtheorem{lemma}[theorem]{Lemma}

\theoremstyle{definition}
\newtheorem{definition}[theorem]{Definition}
\newtheorem{assumption}[theorem]{Assumption}
\theoremstyle{remark}
\newtheorem{remark}[theorem]{Remark}

\title{The Power of Sampling: Dimension-free Risk Bounds in Private ERM}

\author{%
  Yin Tat Lee\thanks{University of Washington and Microsoft Research}\\
  \And Daogao Liu \thanks{University of Washington}
  \And Zhou Lu \thanks{Princeton University}
}

\begin{document}

\maketitle

\begin{abstract}
Differentially private empirical risk minimization (DP-ERM) is a fundamental problem in private optimization. While the theory of DP-ERM is well-studied, as large-scale models become prevalent, traditional DP-ERM methods face new challenges, including (1) the prohibitive dependence on the ambient dimension, (2) the highly non-smooth objective functions, (3) costly first-order gradient oracles.
Such challenges demand rethinking existing DP-ERM methodologies.
In this work, we show that the regularized exponential mechanism combined with existing samplers can address these challenges altogether: under the standard unconstrained domain and low-rank gradients assumptions, our algorithm can achieve rank-dependent risk bounds for non-smooth convex objectives using only zeroth order oracles, which was not accomplished by prior methods. 
This highlights the power of sampling in differential privacy.
We further construct lower bounds, demonstrating that when gradients are full-rank, there is no separation between the constrained and unconstrained settings.
Our lower bound is derived from a general black-box reduction from unconstrained to the constrained domain and an improved lower bound in the constrained setting, which might be of independent interest.

\end{abstract}

\section{Introduction}
Differential privacy, as established in \cite{dmns06}, has become the gold standard for privacy preservation in machine learning. It offers robust guarantees against extracting private individual data from trained models. Specifically, an algorithm $\M$ is said to be $(\epsilon,\delta)$-differentially private\footnote{When $\delta>0$, we refer to it as approximate-DP, and we call the case $\delta=0$ pure-DP.} if for any pair of inputs $\D$ and $\D'$ differing by a single data and any event ${\cal O}\in \mathrm{Range}(\M)$, it satisfies
\begin{align*}
    \Pr[\M(\D)\in {\cal O}]\leq \exp(\epsilon)\Pr[\M(\D')\in {\cal O}] +\delta.
\end{align*}
A pivotal application of DP is in Empirical Risk Minimization (ERM), a fundamental problem in machine learning. In DP-ERM, the goal is to devise a privacy-preserving algorithm that minimizes the loss function 
$$
L(\theta ;\D)=\frac{1}{n}\sum_{i=1}^n \ell(\theta;z_i),
$$
given a family of functions on a domain $\K\subseteq \R^d$ and a dataset $\D=\{z_1,\cdots,z_n\}$. For instance, here $\theta$ can represent the parameters of a neural network, $z_i$ can be a training data pair (image and label), and $\ell(\theta;z_i)$ the classification error for that data.

The quality of the output $\theta$ of a private algorithm is evaluated by its excess empirical loss, defined as $$L(\theta;\D)-\min_{\theta'\in\K}L(\theta';\D),$$  the difference between the loss of $\theta$ and the minimum possible loss over the convex domain $\mathcal{K} \subset \mathbb{R}^d$. In practical terms, this means seeking $\theta$ that minimizes this loss while ensuring as much privacy as possible.

Prior research in DP-ERM has largely focused on \textit{convex} loss functions.
In the most well-studied setting of the constrained domain and Euclidean geometry, a risk bound of $$\Theta\left(\frac{\sqrt{d\log(1/\delta)}}{\epsilon n}\right)$$ is known to be tight \cite{bst14,su15,wyx17,bftt19}. 
However, the polynomial dependency on the dimension $d$ becomes impractical in high-dimensional settings typical of contemporary machine learning, prompting the study of \textit{dimension-free} risk in DP-ERM. We refer to a risk bound as dimension-free (or dimension-independent) if it has no explicit polynomial dependence on the ambient dimension $d$, allowing for dependence on more nuanced properties like the rank of gradient subspaces.

\subsection{Unbounded domain}
Motivated by evading the ambient dimension dependence, there is a line of work \cite{jt14,sstt21,LLH+22} studying how to get 'dimension-free' excess risk bounds and succeed in the unbounded domain when the gradients are low-rank.
We discuss the previous assumptions on the domain and gradients, and the associated interesting findings there.
\begin{assumption}[Constrained Domain]
\label{assump:bounded}
The convex domain $\K\subsetneq\R^d$ of diameter $C$.
\end{assumption}

\begin{assumption}[Unconstrained Domain with Prior Knowledge]
\label{assump:unbounded}
    The convex $\K=\R^d$, and we know there exists $C>0$, such that for any convex loss function $\ell(;z)$ in the universe, the minimizer $\theta^*:=\arg\min_{\theta} \ell(\theta;z)$ satisfies that $\|\theta^*\|\le C$.
\end{assumption}

At first glance, these two assumptions seem equivalent to each other.
For example, restricting the unconstrained domain to a ball of radius $C$, can reduce Assumption~\ref{assump:unbounded} to Assumption~\ref{assump:bounded}.
Though not explicitly straightforward, the reversal direction of reduction is convincing and believable.
Nonetheless, under the low-rank gradients assumption, there is a separation between these two assumptions.

\begin{assumption}[Low-Rank Gradients]
\label{assump:rank-gradient}
There is an orthogonal projection matrix $P$ with rank $\rank$ \footnote{In the classic full-rank assumption, $\rank=d$ and $P=I$.}, such that
\begin{align*}
    \|(I-P)\nabla \ell(\theta;z)\|=0,\forall \theta,\forall z.
\end{align*}
\end{assumption}

Under Assumption~\ref{assump:unbounded} and Assumption~\ref{assump:rank-gradient}, previous work \cite{sstt21} suggests a dimension-independent bound $\Theta(\frac{\sqrt{\rank\log(1/\delta)}}{\epsilon n})$.
On the other hand, under Assumption~\ref{assump:bounded} and Assumption~\ref{assump:rank-gradient}, there is a dimension-dependent lower bound
$\Omega(\frac{\sqrt{d\log(1/\delta)}}{n\epsilon})$, see \cite{bst14}.\footnote{\cite{bst14} does not explicitly state the low-rank gradients assumption, but their lower bound construction is based on GLM, and hence leads to the lower bound claimed above.}
This suggests some deeper differences between the Constrained Domain and the Unconstrained Domain.
For example, if we run DP-SGD under Assumption~\ref{assump:unbounded}, we do not need the projection step, and hence, the noise added vertically to the gradients subspace does not influence the final utility, and we can reduce the problem dimension from $d$ to $\rank$.
This is how some of the previous works got the rank-dependent risk bounds.
On the other hand, if we run DP-SGD under Assumption~\ref{assump:bounded}, we need to project back to $\K$, and the previous analysis does not hold.

Theoretically, we get the dimension-dependent lower bound even for convex loss functions in the classic constrained full-rank setting.
Nonetheless, in practice, large models can be fine-tuned with DP to achieve performance that is approaching that of non-private models.
This contradiction demonstrates the classic assumptions may be too restrictive, and people propose low-rank gradient assumptions as natural relaxations.
We refer the readers to \cite{sstt21,LLH+22} for more justifications about the low-rank gradient assumptions.

\subsection{Motivations}
\cite{jt14} first studied how to achieve dimension-independent risk bounds through the use of output and objective perturbation.
Their bound is suboptimal, and some results rely on the smoothness assumption of the objective functions.
Both \cite{sstt21} and \cite{LLH+22} are based on DP-SGD, with the first-order gradient oracles.
\cite{ZTOH23} is a zeroth order method that assumes the functions are smooth, querying the function values at two near points and using the value difference to estimate the gradients, then applying the gradient descent with the estimated gradients.
In many applications, gradient evaluations can be costly or unavailable; for example, bandit problems, and/or smoothness assumptions may not be feasible.

\begin{quote}
    \textbf{Question 1}: \emph{Can we develop DP-ERM algorithms with dimension-free risk bounds, that do not require smooth loss functions or first-order oracles?}
\end{quote}

We know the low-rank gradients assumption play a crucial role in achieving dimension-independent upper bounds, and with low-rank gradients assumption, there is a separation between the bounded and unbounded domain.
However, does the separation still exist without the low-rank gradients assumptions?
As we discussed, when the gradients are full-rank, we may reduce the problem under the unconstrained assumption to the constrained assumption.
This suggests that Assumption~\ref{assump:unbounded} is a stronger assumption.
It is unclear whether we can get the same lower bounds under Assumption~\ref{assump:bounded} and Assumption~\ref{assump:unbounded}.

\begin{quote}
    \textbf{Question 2}: \emph{Is the lower bound under the assumption of an unconstrained domain (Assumption~\ref{assump:unbounded}) the same as lower bound under constrained domain (Assumption~\ref{assump:bounded}) when the gradients can be full-rank?}
\end{quote}

\subsection{Our contributions}
\paragraph{Question 1:}
We present a positive response to the first question by designing a new algorithm based on the simple exponential mechanism. We show that it can achieve rank-dependent risk bounds in an unconstrained setting for non-smooth convex objectives, using only zeroth-order oracles. This is the first dimension-free result in DP-ERM that neither assumes smoothness nor requires gradient information, aligning more closely with the needs of modern machine learning paradigms. 
In addition, this result is achieved without any algorithmic modifications to the exponential mechanism, illustrating the inherent low-rank property of sampling-based private algorithms.

\paragraph{Question 2:}
In response to the second question, we establish the same lower bound applicable under both domain assumptions.
We establish a general black-box reduction from the unconstrained to the constrained setting.
Our result indicates no separation between the unconstrained domain assumption and the constrained domain assumption with full-rank gradients,
advancing our understanding of dimension-free DP-ERM. Furthermore, our lower bound is broadly applicable and improved over previous results: it's valid across any $\ell_p$ geometry for $p\ge 1$, improving the previously known best lower bound of \cite{afkt21}. For detailed comparisons, we refer to table \ref{tab}.

\begin{table*}[ht] 
\label{tab}
\begin{center}
{
\begin{tabular}{|c|c|c|c|c|c|}
\hline 
Article& Constrained?& $\ell_p$ & Loss Function& Pure DP& Approximate DP\\
\hline 
\cite{bst14}& constrained & $p=2$ & GLM & $\Omega({\frac{d}{n\epsilon}})$ &$\Omega({\frac{\sqrt{d}}{n\epsilon}})$\\
\hline 
\cite{su15}& constrained & $p=2$ & GLM & N/A & $\Omega({\frac{\sqrt{d\log(1/\delta)}}{n\epsilon}})$\\
\hline
\cite{sstt21}& unconstrained & $p=2$ & GLM &  N/A &$\Omega(\frac{\sqrt{\rank}}{{n\epsilon}})$\\
\hline 
\cite{afkt21}&  both& $p=1$ & general& N/A & $\Omega({\frac{\sqrt{d}}{n\epsilon\log d}})$\\
\hline
\cite{bgn21} & constrained & $1<p\leq 2$ & GLM & N/A & $\Omega((p-1)\frac{\sqrt{d\log(1/\delta)}}{n\epsilon})$\\
\hline
Ours & both & $1\leq p\leq \infty$ &  general & $\Omega({\frac{d}{n\epsilon}})$ &$\Omega({\frac{\sqrt{d\log(1/\delta)}}{n\epsilon}})$\\
\hline
\end{tabular}}
\caption{Comparison of lower bounds for private convex ERM. 
\iffalse The lower bound of \cite{sstt21} is weaker than ours in the important over-parameterized $d\gg n$ setting, as $\rank\le \min\{n,d\}$.\fi}
\end{center}
\end{table*}

\subsection{Related work}
The first dimension-independent bounds were achieved by \cite{jt14} through the use of output and objective perturbation. 
Subsequently, \cite{sstt21,LLH+22} improved the results of \cite{jt14} and achieved the dimension-independent bounds by utilizing the DP-SGD. 
The approach of \cite{sstt21} assumes that the function gradients are precisely situated within a low-rank subspace, whereas \cite{LLH+22} relaxed this constraint, allowing gradients to extend outside the low-rank subspace.
We follow the same assumption as \cite{LLH+22} and will specify it later.
Currently, DP-SGD stands as the sole mechanism known to achieve optimal dimension-independent bounds under approximate DP.

The majority of existing lower bounds in DP-ERM utilize GLM functions. As an example, \cite{bst14} employs a linear function, $\ell(\theta;z)=\langle \theta,z\rangle$, that doesn't extend to the unconstrained case due to potential infinite loss values. To address this limitation, \cite{sstt21} adopts the objective functions $\ell(\theta;z)=|\langle \theta,x\rangle-y|$. By transforming the problem of minimizing GLM into estimating the mean of a set of vectors, they derived the lower bound using tools from coding theory.

Works such as \cite{krrt20,zwb20} explored how to circumvent the curse of dimensionality for functions beyond GLMs, employing public data to identify a low-rank subspace, an approach conceptually akin to \cite{sstt21}. Differential Private Stochastic Convex Optimization (DP-SCO) \cite{fkt20,bfgt20,bftt19,kll21,afkt21,bgn21}, a closely associated problem to DP-ERM, seeks to minimize the function $\E_{z\sim {\cal P}}[\ell(\theta;z)]$ given some underlying distribution ${\cal P}$. DP-SCO's tight bound typically constitutes the maximum informational lower bound on (non-private) SCO and the lower bound on DP-ERM, so improved lower bounds on DP-ERM can further enhance DP-SCO research.

There has been emerging interest in DP-ERM within non-Euclidean settings. Most prior studies considered the constrained Euclidean context, where the convex domain and (sub)gradients of objective functions possess bounded $\ell_2$ norms. In contrast, DP-ERM concerning the general $\ell_p$ norm is relatively under-explored. Driven by the significance and broad applicability of non-Euclidean settings, prior works \cite{ttz15,afkt21,bgn21,bgm21,HLL+22,GLLST23} have scrutinized constrained DP-ERM and DP-SCO with respect to the general $\ell_p$ norm, yielding a myriad of intriguing results. 
However, there are still gaps between the current upper and lower bounds demonstrated in the paper when $p>2$.

Recently, driven by the need for private fine-tuning of large models, research has shifted towards differentially private algorithms employing zeroth-order oracles. \cite{ZTOH23} investigated the private minimization of gradient norms for non-convex smooth objectives with function value evaluations under a modified low-rank assumption.
\cite{TPN+24} proposed a DP-ERM algorithm with zeroth order oracles but only analyzed its privacy guarantee and empirical performance without theoretical risk bounds.

\section{Rank-dependent upper bound via sampling}\label{sec sam}



We present the rank-dependent upper bound by sampling from exponential mechanism in this section.
Our approach is grounded in the following standard assumption on the low-rank structure of objective functions, which is employed by \cite{LLH+22}:

\begin{assumption}[Restricted Lipschitz Continuity]
\label{assm:rank}
For any $s$, $\ell(\theta;z)$ is convex and $G$-Lipschitz over $\theta\in\R^d$.
For each $k\in[d]$,
there exists an orthogonal projection matrix $P_k$ with rank $k$ such that 
\begin{align*}
    \|(I-P_k)\nabla \ell(\theta;z)\|_2\le G_k ,\forall \theta,\forall z,
\end{align*}
where the (sub)gradient is taken over $\theta$.
\end{assumption}
It is evident that $G=G_0\ge G_1\ge\cdots\ge G_d=0$.
An example of $P_k$ is a diagonal matrix such that the first $k$ diagonal entries are 1, and others are 0.
This means the $\ell_2$ norm of the last $d-k$ dimensions of $\nabla \ell(\theta;z)$ is bounded by $G_k$.

\cite{sstt21} introduced the low-rank assumption which is equivalent to assuming $G_{\mathrm{rank}}=0$. This assumption, however, was later recognized as potentially overly restrictive. 
Consequently, it was relaxed to a more flexible version, i.e., Assumption~\ref{assm:rank} by \cite{LLH+22}. 
To substantiate this relaxed assumption, \cite{LLH+22} conducted multiple experiments, including Principal Component Analysis (PCA) and fine-tuning models within the principal subspace of reduced dimensions, demonstrating that these models can achieve performance comparable to their original higher-dimensional counterparts. 
We direct readers to the work of \cite{LLH+22} for a comprehensive discussion of the assumption and findings.

\begin{algorithm}[tb]
\caption{The Regularized Exponential Mechanism} 
\begin{algorithmic}
\STATE {\bf Inputs:}  parameters $\epsilon,\delta,C$, Restricted Lipschitz Continuity parameters $\{G_k\}$, dataset $\D$
\STATE Set $\eta=\frac{n\epsilon\sqrt{k\log(1/\delta)}}{GC},\mu=\frac{8\eta G^2}{n^2\epsilon^2}$
\STATE Sample $\theta^{app}$ prop to $\exp(-\eta(L(\theta;\D)+\frac{\mu}{2}\|\theta\|_2^2))$
\STATE {\bf Output:} $\theta^{app}$
\end{algorithmic} 
\label{alg exp mec}
\end{algorithm}






We then present our main upper bound result, an $\tilde{O}(\frac{\sqrt{\text{rank}}}{n\epsilon})$ risk bound that matches those of \cite{sstt21,LLH+22}.  

\begin{theorem}[Approximate-DP]
\label{thm main app}
    Under Assumption~\ref{assump:unbounded} and Assumption~\ref{assm:rank}, for $\epsilon,\delta\in (0,1/2)$, if for some $k\in[d]$ such that $G_k\le \frac{G}{n\epsilon\sqrt{d}}$, setting $\eta=\frac{ n \epsilon\sqrt{k\log(1/\delta)}}{GC}$ and $\mu=\frac{8\eta G^2}{n^2\epsilon^2}$, sampling $\theta^{app}$ with probability proportional to $\exp(-\eta(L(\theta;\D)+\mu\|\theta\|^2_2/2))$ as in Algorithm \ref{alg exp mec} is $(\epsilon,\delta)$-DP, and
    \begin{align*}
        \E[L(\theta^{app};\D)-L(\theta^*;\D)]\lesssim \frac{GC\sqrt{k\log(1/\delta)}}{n\epsilon},
    \end{align*}
    where $\theta^*=\arg\min L(\theta;\D)$. In particular, in expectation, only $O(n^2\epsilon^2\log^2(nd/\delta))$ calls to the zero-th order oracle is required.
\end{theorem}

The risk bounds in the above theorem is dimension-free, depending on the rank $k$ instead of the ambient dimension $d$. Meanwhile, Algorithm \ref{alg exp mec} uses only zero-th order oracles and doesn't require smoothness of the functions. As a comparison, \cite{sstt21,LLH+22} are both based on DP-SGD and require first-order oracles, while \cite{ZTOH23} targets a different problem and requires smoothness of objectives. In addition, our algorithm is efficient to implement as well for its $\Tilde{O}(n^2\epsilon^2)$ oracle complexity.

The privacy guarantee and computation complexity are mostly based on previous work \cite{gll22}, which studies regularized exponential mechanisms in the classic setting: constrained domain with full-rank gradients.

The challenge is demonstrating the utility bound.
Our method is based on analyzing the variance of the sampling method.
If we sample $x$ from the distribution $\pi(x)\propto \exp(-\eta f(x))$ for some convex function $f$ under the low-rank assumption (Assumption~\ref{assump:rank-gradient}), it is straightforward to show $\E_{x\sim \pi}f(x)-f(x^*)\le \rank/\eta$.
However, if we relax the low-rank assumption to Restricted Lipschitz Continuity (Assumption~\ref{assm:rank}), the trivial argument does not work directly.
Moreover, to make the mechanism satisfy the approximate DP, we need to add some strongly convex regularizer to the objective function, as demonstrated in Algorithm~\ref{alg exp mec}.

Lemma~\ref{lm:sampling_utility} is our main technical lemma to bound the utility.
We begin with a helpful lemma on the intrinsic property of the sampling method.

\begin{restatable}{lemma}{errbyvar}
    \label{lm:err_by_var}
    For a convex function $f$ with global minimum point $x^*$, let $\pi$ be the distribution proportional to $\exp(-f(x)-\frac{\mu}{2}\|x\|^2)$. Then we have 
\begin{align*}
    \E_{x\sim \pi}f(x)=f(x^*)+\int_{1}^{\infty}\Var_{x\sim \pi_t}(f(x))\d t,
\end{align*}
where $\Var_{x\sim \pi_t}$ is the variance under the distribution $\pi_t\propto\exp(-tf(x)-\frac{\mu}{2}\|x\|^2)$.
\end{restatable}

As a result, to get the utility guarantee of the sampling mechanism, it suffices to bound the variance $\Var_{x\sim \pi_t}(f)$.
The standard approach for bounding the variance, unfortunately, involves dependence on dimension:

\begin{lemma}[Theorem 3 in \cite{Che21}]
\label{lm:var_by_dim}
Let $f$ be a convex function on $\R^d$ and $\pi$ be the distribution proportional to $\exp(-f(x))$, then we have
\begin{align*}
    \Var_{x\sim\pi}f(x)\leq d.
\end{align*}
\end{lemma}

To ensure the objective density is well-defined in the unconstrained case (whose support is the whole space $\R^d$), we add a regularizer term, and bound the variance under this regularized strongly log-concave density.
\begin{restatable}{lemma}{sampleone}
\label{lm:regular_var_by_dim}
Let $\pi$ be the distribution given by $\exp(-f(x)-\frac{\mu}{2}\|x\|^2)$ on $\R^d$. One has
\begin{align*}
    \Var_{x\sim \pi}f(x)\leq 4d+\frac{\mu}{2}\|\ox\|^2,
\end{align*}
where $\ox=\E_{x\sim \pi}x$.
\end{restatable}

There is a dimension dependence in Lemma~\ref{lm:regular_var_by_dim}, which is undesirable.
To fully eliminate the dimension dependence in Lemma~\ref{lm:regular_var_by_dim}, we first derive a new lemma that bounds the variance by dimension and gradient.
It is standard to bound the term $\E_{x\sim\pi}\|x-x^*\|_2^2$ by $d/\mu$, for example, see \cite{DM16} and references therein.
We modify the previous lemmas and bound $\E_{x\sim\pi}\|Q(x-x^*)\|^2_2$ instead.

\begin{restatable}{lemma}{sampletwo}
\label{lm:bound_first_k_dim}
Let $x^*=\arg\min_{x}f(x)+\frac{\mu}{2}\|x\|^2_2$ and $\pi$ be the distribution proportional to $\exp(-f(x)-\frac{\mu}{2}\|x\|^2_2)$.
Letting $Q$ be the projection matrix to the first $k$ coordinates,
we have
\begin{align*}
    \E_{x\sim\pi}\|Q(x-x^*)\|^2_2\leq k/\mu.
\end{align*}
\end{restatable}

For simplicity, we use $a\lesssim b$ to represent that $a=O(b)$ in the following statements.
Recall Assumption \ref{assm:rank},
by rotating the space, we can rewrite $x=(x_{1},x_{2})$ where $x_{1}\in\R^{k}$ and $x_{2}\in\R^{d-k}$ and that $\|\nabla_{2}f(x)\|_{2}\leq G_{k}\text{ for all }x$, where $\nabla_2$ is the gradient on the direction of the block $x_2$.

We decompose the variance $\Var_{(x_1,x_2)\sim\pi}f(x)$ as
\begin{align*}
\E_{x_2\sim\pi}\Var_{x_1\mid x_2\sim\pi}f(x)
    +\Var_{x_2\sim\pi}(\E_{x_1\mid x_2\sim\pi}f(x)),
\end{align*}
where $x_1\mid x_2$ means the distribution of $x_1$ conditional on $x_2$, which is $k$-dimensional.
Hence we can bound the first term, $\Var_{x_1\mid x_2}$ with dependence on $k$.
Through a careful analysis which demonstrates the second term is zero, we get the following rank-dependent bound on variance.

\begin{restatable}{lemma}{samplethree}
\label{lm:var_k}
Suppose $f(x)$ is convex and satisfies Assumption \ref{assm:rank}, and suppose $\pi$ is the distribution proportional to $\exp(-f(x)-\frac{\mu}{2}\|x\|^2_2)$,
we have that
\begin{align*}
    \Var_{x\sim\pi}f(x)\lesssim
    (\frac{G_k^2}{\mu}+1)(k+\mu\|x^*\|^2_2),
\end{align*}
where $x^*=\arg\min_{x}f(x)$.
\end{restatable}

Applying Lemma~\ref{lm:err_by_var} and Lemma~\ref{lm:var_k}, it is immediate to get the key technical lemma.
\begin{restatable}{lemma}{samplefour}
\label{lm:sampling_utility}
Given $t>0$ and let $p(x)$ be the distribution proportional to $\exp\big(-\eta(f(x)+\frac{\mu}{2}\|x\|^2_2)\big)$, we have
\begin{align*}
    &\E_{x\sim p}f(x)-\min_{x}f(x)\lesssim  
    \mu\|x^*\|^2+ \int_1^{\infty}\min_k\Big\{\frac{G_k^2}{\mu}(k+\eta\mu\cdot\|x^*\|^2)+\frac{k}{\eta t^2}\Big\}\d t.
\end{align*}

where $x^*=\arg\min_{x}f(x)$.
\end{restatable}
The utility guarantee of Theorem~\ref{thm main app} follows directly from Lemma~\ref{lm:sampling_utility}.
Basically, when $G_k$ is small enough, then the error term depending on $G_k$ will be negligible, and we get the optimal excess risk bound.
We defer the omitted proof to the Appendix~\ref{appendix:omitted_proof_upper_bound}.




\section{Lower bound for the unconstrained setting}
In the study of dimension-free risk in DP-ERM, much of the focus has been on establishing positive results, particularly in the form of upper bounds like those presented in this work and others \cite{sstt21,LLH+22}. However, to fully grasp the scope and limitations of dimension-free risk bounds, it's essential to investigate both their potential and inherent constraints. Particularly, existing upper bounds, including our own, rely on two key assumptions: (1) low-rank gradients (Restricted Lipschitz Continuity); (2) unconstrained domain, to evade the $\sqrt{d}$ dependence in the constrained setting. 

We now turn our attention to examining the role of the unconstrained domain assumption, by showing that there is no separation between the constrained and unconstrained domain assumptions when the gradients are full-rank.
Formally, we have the following lower bound for the unconstrained setting:

\begin{restatable}{theorem}{appmain}
\label{thm app main}
Let $n,d$ be large enough and $1\geq \epsilon>0,2^{-O(n)}<\delta<o(1/n)$ and $p\ge 1$. 
There exists $G$-Lipschitz convex loss functions $\ell$, such that for every $(\epsilon,\delta)$-differentially private algorithm with output $\theta^{priv}\in \mathbb{R}^d$,  there is a data-set $\D=\{z_1,...,z_n\}\subset \{0,1\}^d \cup \{\frac{1}{2}\}^d$ such that
$$
    \E[L(\theta^{priv};\D)-L(\theta^{\star};\D)]=\Omega(\min(1,\frac{\sqrt{d\log(1/\delta)} }{n\epsilon }) GC),
$$
where $\theta^{\star}$ is a minimizer of $L(\theta;\D)$ and $C=\|\theta^{\star}\|$. Both $G,C$ are defined w.r.t any $\ell_p$ geometry with $p\ge 1$.
\end{restatable}

We obtain this result by a general black-box reduction method. In addition to the applicability to the unconstrained case, our bound is also stronger than previous ones and can be applied to general $\ell_p$ geometry.

Theorem \ref{thm app main} is a direct consequence of two separate results (Theorem~\ref{thm4} and Theorem~\ref{thm2}), detailed in the following subsections. The first part is the black-box reduction from the unconstrained case to the constrained case. Via an extension of Lipshitz convex functions from constrained to unconstrained domain, we show that DP-ERM on the extended function is as hard as the original one.

The second part is an improved lower bound in the constrained setting. For the lower bound construction, we use an $\ell_{\infty}$ ball as the domain and select the $\ell_1$ loss function $\ell(\theta;z)=|\theta-z|_1$, and improve the previous lower bound via the group privacy technique. The choice of the norms on the domain and loss function makes it applicable for general $\ell_p$ geometry with $p\ge 1$.

\subsection{General lower bound by reduction}\label{sec:reduction}
In this section, we present a general black-box reduction method that effectively extends any DP-ERM risk lower bound from a constrained scenario to an unconstrained one. As a case in point, which we detail in the appendix, we utilize our reduction approach to obtain a pure-DP lower bound in the unconstrained setting from the constrained case result \cite{bst14}.

Our result relies on the following key lemma from \cite{cobzas1978norm}, which provides a Lipschitz extension of any convex Lipschitz function from a bounded convex set to the entirety of the domain $\R^d$.

\begin{lemma}[Theorem 1 in \cite{cobzas1978norm}]\label{lem4}
Let $f$ be a convex function which is $\eta$-Lipschitz w.r.t. $\ell_2$ and  defined on a convex bounded set $\K\subset \R^d$.
Define an auxiliary function $g_y(x)$ as:
\begin{equation}
    g_y(x):=f(y)+\eta \|x-y\|_2, y\in \K, \forall x\in \R^d.
\end{equation}
Then consider the function $\tilde{f}:\R^d\rightarrow \R$ defined as $\tilde{f}(x):=\min_{y\in \K} g_y(x)$.
We know $\tilde{f}$ is $\eta$-Lipschitz w.r.t. $\ell_2$ on $\R^d$, and $\tilde{f}(x)=f(x)$ for any $x\in\K$.
\end{lemma}

For any $y\in\R^d$, we define $\Pi_{\K}(y):=\arg\min_{x\in \K}\|x-y\|_2$.
It is well-known in the convex analysis, that for a compact convex set $\K$ and any point $y\in\R^d$, the the set $\{x\in\K:\|x-y\|_2<\|z-y\|_2,\forall z\in \K,z\ne x\}$ is always non-empty and singleton \cite{hazan2019introduction}.

The main idea of our reduction result is that, we can extend the ``hard" loss function for any lower bound in the constrained setting to $\R^d$ using the above lemma, then show the same bound still holds. An important observation on such convex extension is that the loss $L(\theta;\D)$ value at a point $\theta$ does not increase after projecting $\theta$ onto the convex domain $\K$, i.e., $L(\theta;\D) \ge L(\Pi_{\K}(\theta);\D)$. This property can be derived from the Pythagorean Theorem (Lemma~\ref{lm:Pythagorean}) for any convex set, in combination with the specific structure of the extension.


We define a 'witness function' for any lower bound in the constrained setting, to serve as the black-box. For example, in \cite{bst14} the (witness) loss function is simply linear and the lower bound is roughly $\Omega(\min \{1, \frac{\sqrt{d}}{n\epsilon}\})$.

\begin{definition}\label{prop1}
Let $n,d$ be large enough, $0\leq \delta\leq 1$ and $\epsilon>0$. We say functions $\ell$ is a witness to the lower bound function $f$, if for any $(\epsilon,\delta)$-DP algorithm,
there exist a convex set $\K\subset \R^d$ of diameter $C$, a family of $G$-Lipschitz convex functions $\ell(\theta;z)$ defined on $\K$ w.r.t. $\ell_2$, a dataset $\D$ of size $n$, such that with probability at least $1/2$ (over the random coins of the algorithm),
\begin{align*}
    L(\theta^{priv};\D)-\min_{\theta\in \cK}L(\theta;\D)=\Omega(f(d,n,\epsilon,\delta,G,C)),
\end{align*}
where $L(\theta;\D):=\frac{1}{n}\sum_{i=1}^{n}\ell(\theta;z_i)$ and $\theta^{priv}\in \K$ is the output of the algorithm.
\end{definition}
The function $f$ can be any lower bound in the constrained case with dependence on the parameters, and $\ell$ is the loss function used to construct the lower bound. 
We use the Lipschitz extension mentioned above to define our new loss function in the unconstrained case, i.e.,
\begin{equation}
    \tilde{\ell}(\theta;z)=\min_{y\in \K} \ell(y;z) +G\|\theta-y\|_2
\end{equation}
which is convex, G-Lipschitz and equal to $\ell(\theta;z)$ when $\theta\in \K$ by Lemma~\ref{lem4}. 
Our intuition is simple: if $\theta^{priv}$ lies in $\K$, then we are done by using the witness function and lower bound from Definition \ref{prop1}. If not, the projection of $\theta^{priv}$ to $\K$ should lead to a smaller loss. However, the projected point cannot have a minimal loss due to the lower bound in Definition \ref{prop1}, let alone $\theta^{priv}$ itself. As a consequence, we obtain the following theorem on the reduction from unconstrained to constrained.

\begin{restatable}{theorem}{reductionthm}
\label{thm4}
Assume $\ell, f$ are the witness function and lower bound as in Definition \ref{prop1}.
For any $(\epsilon,\delta)$-DP algorithm and any initial point $\theta_0\in\R^d$, there exist a family of $G$-Lipschitz convex functions $\tilde{\ell}(\theta;z):\R^d\rightarrow \R$ being the $\ell$ from Definition \ref{prop1}, a dataset $\D$ of size n and the same function $f$, such that with probability at least 1/2 (over the random coins of the algorithm)
\begin{equation}
    \tilde{L}(\theta^{priv};\D)-\tilde{L}(\theta^*;\D)=\Omega(f(d,n,\epsilon,\delta,G,C)),
\end{equation}
where $\tilde{L}(\theta;\D):=\frac{1}{n}\sum_{z_i\in\D}\tilde{\ell}(\theta;z_i)$ is the ERM objective function, $\theta^*=\arg\min_{\theta\in\R^d}\tilde{L}(\theta;\D)$, $C\geq \|\theta_0-\theta^*\|_2$ and $\theta^{priv}$ is the output of the algorithm.
\end{restatable}

Theorem \ref{thm4} shows that unconstrained DP-ERM is as hard as its constrained counterpart, and as a result it's impossible to achieve dimension-independent upper bounds in general without further assumptions. As an example, the low-rank Assumption \ref{assm:rank} is essential to our rank-dependent upper bound Theorem \ref{thm main app}.

\subsection{Improved lower bound}\label{sec:improvebound}
In this part, we improve the lower bounds for approximate DP. Our goal is twofold: to tighten the previous lower bounds and to extend this boundary to encompass any non-euclidean geometry and the unconstrained case. We assume that $2^{-O(n)}<\delta<o(1/n)$. The supposition concerning $\delta$ is standard in the literature, as seen, for instance, in \cite{su15}.

\paragraph{Motivation and main idea}
Previous works in the constrained case \cite{bst14,su15}  fail in the unconstrained and non-euclidean case for two reasons.
First, they rely on the $\ell_2$ ball as the domain, which lacks the generalizability to the general $\ell_p$ norm.
Second, to generalize the lower bound to the unconstrained case, linear functions are no longer appropriate to be loss functions, as they can take minus infinity values and lack a global minimum.

To circumvent these issues, we consider an $\ell_\infty$ ball as the domain and select the loss function $\ell(\theta;z)=|\theta-z|_1$. Formally, the loss function is defined as follows:
\begin{align*}
    \ell(\theta;z)=\|\theta-z\|_1, \theta\in\R^d,z\in \{-1,1\}^d.
\end{align*}
The convex domain $\K$ is the $\ell_\infty$ unit ball. For any data-set $\D=\{z_1,...,z_n\}$, the loss function is
\begin{align*}
    L(\theta;\D)=\frac{1}{|\cD|}\sum_{i=1}^{|\cD|} \ell(\theta;z_i)=\frac{1}{|\cD|}\sum_{i=1}^{|\cD|}\|\theta-z_i\|_1.
\end{align*} 
Our rationale for this choice is twofold. Firstly, $\ell_1$ and $\ell_{\infty}$ serve as the "strongest" norms for loss and domain, respectively, implying lower bounds for general $\ell_p$ geometry by the Holder inequality. Secondly, the $\ell_1$ loss function can be directly generalized to the unconstrained case.

The technical difficulty of the unconstrained case lies in the fact that we can no longer straightforwardly reduce the lower bound of the DP-ERM to the lower bound of mean estimation, a strategy adopted by previous works. Specifically, a large mean estimation error does not necessarily result in a large empirical risk.

Consider a simple example. Recall that we want to minimize $L(\theta;\D)=\sum_{i=1}^{n}\ell(\theta;z_i)/n$ over the $\ell_\infty$ unit ball $\K$, where $\ell(\theta;z)=\|\theta-z\|_1$ and each $z_i\in\{0,1\}^d$ as the set up before.
If $\frac{1}{n}\sum_{i=1}^{n}z_i=\frac{1}{2}\mathbf{1}$ where $\mathbf{1}$ is the all-one vector, then $L(\theta;\D)$ is a constant function, equal to $d/2$ for any $\theta\in\K$.
In this example, for a bad estimator $\theta_{\mathrm{bad}}$, even if $\|\theta_{\mathrm{bad}}-\frac{1}{n}\sum_{i=1}^{n}z_i\|_2$ is large, it can still be a minimizer to the loss function, i.e., $L(\theta_{\mathrm{bad}};\cD)-\min_{\theta\in\K}L(\theta;\cD)=0$.

\paragraph{Main result in Euclidean geometry}
Similar to \cite{buv18}, we have the following standard lemma, which allows us to reduce any $\epsilon<1$ to the $\epsilon=1$ case without losing generality. The proof is based on the well-known 'secrecy of the sample' lemma from \cite{kasiviswanathan2011can}.

\begin{restatable}{lemma}{apptwo}
\label{lemep}
For $0<\epsilon<1$, a condition $Q$ has sample complexity $n^*$ for algorithms with $(1,o(1/n))$-differential privacy ($n^*$ is the smallest sample size that there exists an $(1,o(1/n))$-differentially private algorithm $\A$ which satisfies $Q$), if and only if it also has sample complexity $\Theta(n^*/\epsilon)$ for algorithms with $(\epsilon,o(1/n))$-differential privacy.
\end{restatable}

We apply the group privacy technique in \cite{su15}, based on the following technical lemma:

\begin{restatable}{lemma}{error_by_appending}
\label{lm:error_by_appending}
Let $n,k$ be two large positive integers such that $k<n/1000$.
Let $n_k=\lfloor n/k \rfloor$.
Let $z_1,\cdots,z_{n_k}$ be $n_k$ numbers where $z_i\in\{0,1,1/2\}$ for all $i\in[n_k]$.
For any real value $q\in[0,1]$, if we copy each $z_i$ $k$ times, and append $n-kn_k$ '0' to get $n$ numbers $z_1',\cdots,z_n'$, then we have
\begin{align*}
    |\sum_{i=1}^{n_k}|q-z_i|/n_k-\sum_{i=1}^{n}|q-z_i'|/n|\le 3k/n.
\end{align*}    
\end{restatable}
This lemma bounds the average absolute distance of $q$ between $\{z_i\}$ and $\{z_i'\}$. 
For the construction of our lower bound, we will copy a small dataset a few times and append '0' via this lemma.

The following theorem presents the improved lower bound we obtain, which modifies and generalizes the techniques in \cite{su15,bst14} to reach a tighter bound for the unconstrained case. 

\begin{restatable}[Lower bound for $(\epsilon,\delta)$-differentially private algorithms]{theorem}{appthree}
\label{thm2}
Let $n,d$ be large enough and $1\geq \epsilon>0,2^{-O(n)}<\delta<o(1/n)$. For every $(\epsilon,\delta)$-differentially private algorithm with output $\theta^{priv}\in \R^d$,  there is a data-set $\D=\{z_1,...,z_n\}\subset \{0,1\}^d \cup \{\frac{1}{2}\}^d$ such that
\begin{equation}
    \E[L(\theta^{priv};\D)-L(\theta^{\star};\D)]=\Omega(\min(1,\frac{\sqrt{d\log(1/\delta)} }{n\epsilon }) GC)
\end{equation}

where $\ell$ is G-Lipschitz w.r.t. $\ell_2$ geometry, $\theta^{\star}$ is a minimizer of $L(\theta;\D)$, and $C=\sqrt{d}$ is the diameter of $\K$ w.r.t. $\ell_2$ geometry,
where $\K$ is the unit $\ell_\infty$ ball containing all possible true minimizers and differs from its usual definition in the constrained setting.
\end{restatable}

\begin{remark}
The dependence on parameters $GC$ is standard. For example, one can scale the loss function to be $\hat{\ell}(x;z)=\|ax-z\|_1$ for some constant $a\in (0,1)$, which decreases Lipschitz constant $G$ but increases the diameter $C$ (we should choose $\K$ to contain all possible minimizes).
\end{remark}
 
This bound improves a log factor over \cite{bgn21} and can be directly extended to the constrained bounded setting, by setting the constrained domain to be the unit $\ell_\infty$ ball.
\paragraph{Extension to non-Euclidean geometry}
We illustrate the power of our construction in Theorem~\ref{thm2},
by showing that the same bound holds for any $\ell_p$ geometry where $p\ge 1$ in the constrained setting, and the bound is tight 
for all $1< p\le 2$, improving/generalizing existing results in \cite{afkt21,bgn21}. 

Our construction is advantageous in that it uses $\ell_1$ loss and $\ell_{\infty}$-ball-like domain in the constrained setting, both being the strongest in their direction when relaxing to $\ell_p$ geometry. Simply using the Holder inequality yields that the product of the Lipschitz constant $G$ and the diameter of the domain $C$ is equal to $d$ when $p$ varies in $[1,\infty)$.

\begin{restatable}{theorem}{appfour}
\label{thm3}
Let $n,d$ be large enough and $1\geq \epsilon>0,2^{-O(n)}<\delta<o(1/n)$ and $p\ge 1$. 
There exists a convex set $\K\subset\R^d$, such that for every $(\epsilon,\delta)$-differentially private algorithm with output $\theta^{priv}\in \K$,  there is a data-set $\D=\{z_1,...,z_n\}\subset \{0,1\}^d \cup \{\frac{1}{2}\}^d$ such that
\begin{equation}
    \E[L(\theta^{priv};\D)-L(\theta^{\star};\D)]=\Omega(\min(1,\frac{\sqrt{d\log(1/\delta)} }{n\epsilon }) GC),
\end{equation}
where $\theta^{\star}$ is a minimizer of $L(\theta;\D)$, $\ell$ is G-Lipschitz, and $C$ is the diameter of the domain $\K$. Both $G$ and $C$ are defined w.r.t. $\ell_p$ geometry.
\end{restatable}

For the unconstrained case, we notice that the optimal $\theta^*$ under our construction must lie in the unit $\ell_{\infty}$-ball $\K=\{x\in \R^d| 0\le x_i\le 1, \forall i\in [d]\}$, by observing that projecting any point to $\K$ does not increase the $\ell_1$ loss. 
Therefore, our result can be generalized to the unconstrained case directly.
In a word, our result presents lower bounds $\Omega(\frac{\sqrt{d\log(1/\delta)}}{\epsilon n})$ for all $p\ge 1$ and for both constrained case and unconstrained case. 

\bibliographystyle{plainnat}
\bibliography{example_paper}

\begin{thebibliography}{34}
\providecommand{\natexlab}[1]{#1}
\providecommand{\url}[1]{\texttt{#1}}
\expandafter\ifx\csname urlstyle\endcsname\relax
  \providecommand{\doi}[1]{doi: #1}\else
  \providecommand{\doi}{doi: \begingroup \urlstyle{rm}\Url}\fi

\bibitem[Asi et~al.(2021)Asi, Feldman, Koren, and Talwar]{afkt21}
Hilal Asi, Vitaly Feldman, Tomer Koren, and Kunal Talwar.
\newblock Private stochastic convex optimization: Optimal rates in $\ell_1$
  geometry.
\newblock \emph{arXiv preprint arXiv:2103.01516}, 2021.

\bibitem[Bassily et~al.(2014)Bassily, Smith, and Thakurta]{bst14}
Raef Bassily, Adam Smith, and Abhradeep Thakurta.
\newblock Private empirical risk minimization: Efficient algorithms and tight
  error bounds.
\newblock In \emph{2014 IEEE 55th Annual Symposium on Foundations of Computer
  Science}, pages 464--473. IEEE, 2014.

\bibitem[Bassily et~al.(2019)Bassily, Feldman, Talwar, and Thakurta]{bftt19}
Raef Bassily, Vitaly Feldman, Kunal Talwar, and Abhradeep~Guha Thakurta.
\newblock Private stochastic convex optimization with optimal rates.
\newblock In \emph{Advances in Neural Information Processing Systems}, pages
  11282--11291, 2019.

\bibitem[Bassily et~al.(2020)Bassily, Feldman, Guzm{\'a}n, and Talwar]{bfgt20}
Raef Bassily, Vitaly Feldman, Crist{\'o}bal Guzm{\'a}n, and Kunal Talwar.
\newblock Stability of stochastic gradient descent on nonsmooth convex losses.
\newblock \emph{arXiv preprint arXiv:2006.06914}, 2020.

\bibitem[Bassily et~al.(2021{\natexlab{a}})Bassily, Guzm{\'a}n, and
  Menart]{bgm21}
Raef Bassily, Crist{\'o}bal Guzm{\'a}n, and Michael Menart.
\newblock Differentially private stochastic optimization: New results in convex
  and non-convex settings.
\newblock \emph{Advances in Neural Information Processing Systems}, 34,
  2021{\natexlab{a}}.

\bibitem[Bassily et~al.(2021{\natexlab{b}})Bassily, Guzm{\'a}n, and
  Nandi]{bgn21}
Raef Bassily, Crist{\'o}bal Guzm{\'a}n, and Anupama Nandi.
\newblock Non-euclidean differentially private stochastic convex optimization.
\newblock \emph{arXiv preprint arXiv:2103.01278}, 2021{\natexlab{b}}.

\bibitem[Boneh and Shaw(1998{\natexlab{a}})]{boneh1998collusion}
Dan Boneh and James Shaw.
\newblock Collusion-secure fingerprinting for digital data.
\newblock \emph{IEEE Transactions on Information Theory}, 44\penalty0
  (5):\penalty0 1897--1905, 1998{\natexlab{a}}.

\bibitem[Boneh and Shaw(1998{\natexlab{b}})]{bs98}
Dan Boneh and James Shaw.
\newblock Collusion-secure fingerprinting for digital data.
\newblock \emph{IEEE Transactions on Information Theory}, 44\penalty0
  (5):\penalty0 1897--1905, 1998{\natexlab{b}}.

\bibitem[Bun et~al.(2018)Bun, Ullman, and Vadhan]{buv18}
Mark Bun, Jonathan Ullman, and Salil Vadhan.
\newblock Fingerprinting codes and the price of approximate differential
  privacy.
\newblock \emph{SIAM Journal on Computing}, 47\penalty0 (5):\penalty0
  1888--1938, 2018.

\bibitem[Chewi(2021)]{Che21}
Sinho Chewi.
\newblock The entropic barrier is $ n $-self-concordant.
\newblock \emph{arXiv preprint arXiv:2112.10947}, 2021.

\bibitem[Cobzas and Mustata(1978)]{cobzas1978norm}
S~Cobzas and C~Mustata.
\newblock Norm-preserving extension of convex lipschitz functions.
\newblock \emph{J. Approx. Theory}, 24\penalty0 (3):\penalty0 236--244, 1978.

\bibitem[Durmus and Moulines(2016)]{DM16}
Alain Durmus and Eric Moulines.
\newblock Sampling from strongly log-concave distributions with the unadjusted
  langevin algorithm.
\newblock \emph{arXiv preprint arXiv:1605.01559}, 2016.

\bibitem[Dwork et~al.(2006)Dwork, McSherry, Nissim, and Smith]{dmns06}
Cynthia Dwork, Frank McSherry, Kobbi Nissim, and Adam Smith.
\newblock Calibrating noise to sensitivity in private data analysis.
\newblock In \emph{Theory of cryptography conference}, pages 265--284.
  Springer, 2006.

\bibitem[Dwork et~al.(2014)Dwork, Roth, et~al.]{dwork2014algorithmic}
Cynthia Dwork, Aaron Roth, et~al.
\newblock The algorithmic foundations of differential privacy.
\newblock \emph{Foundations and Trends in Theoretical Computer Science},
  9\penalty0 (3-4):\penalty0 211--407, 2014.

\bibitem[Feldman et~al.(2020)Feldman, Koren, and Talwar]{fkt20}
Vitaly Feldman, Tomer Koren, and Kunal Talwar.
\newblock Private stochastic convex optimization: optimal rates in linear time.
\newblock In \emph{Proceedings of the 52nd Annual ACM SIGACT Symposium on
  Theory of Computing}, pages 439--449, 2020.

\bibitem[Gopi et~al.(2022)Gopi, Lee, and Liu]{gll22}
Sivakanth Gopi, Yin~Tat Lee, and Daogao Liu.
\newblock Private convex optimization via exponential mechanism.
\newblock \emph{arXiv preprint arXiv:2203.00263}, 2022.

\bibitem[Gopi et~al.(2023)Gopi, Lee, Liu, Shen, and Tian]{GLLST23}
Sivakanth Gopi, Yin~Tat Lee, Daogao Liu, Ruoqi Shen, and Kevin Tian.
\newblock Private convex optimization in general norms.
\newblock In \emph{Proceedings of the 2023 Annual ACM-SIAM Symposium on
  Discrete Algorithms (SODA)}, pages 5068--5089. SIAM, 2023.

\bibitem[Han et~al.(2022)Han, Liang, Liang, Wang, Yao, and Zhang]{HLL+22}
Yuxuan Han, Zhicong Liang, Zhipeng Liang, Yang Wang, Yuan Yao, and Jiheng
  Zhang.
\newblock Private streaming sco in $\ell\_p$ geometry with applications in high
  dimensional online decision making.
\newblock In \emph{International Conference on Machine Learning}, pages
  8249--8279. PMLR, 2022.

\bibitem[Hazan(2019)]{hazan2019introduction}
Elad Hazan.
\newblock Introduction to online convex optimization.
\newblock \emph{arXiv preprint arXiv:1909.05207}, 2019.

\bibitem[Jain and Thakurta(2014)]{jt14}
Prateek Jain and Abhradeep~Guha Thakurta.
\newblock (near) dimension independent risk bounds for differentially private
  learning.
\newblock In \emph{International Conference on Machine Learning}, pages
  476--484. PMLR, 2014.

\bibitem[Kairouz et~al.(2020)Kairouz, Ribero, Rush, and Thakurta]{krrt20}
Peter Kairouz, M{\'o}nica Ribero, Keith Rush, and Abhradeep Thakurta.
\newblock Dimension independence in unconstrained private erm via adaptive
  preconditioning.
\newblock \emph{arXiv preprint arXiv:2008.06570}, 2020.

\bibitem[Kasiviswanathan et~al.(2011)Kasiviswanathan, Lee, Nissim,
  Raskhodnikova, and Smith]{kasiviswanathan2011can}
Shiva~Prasad Kasiviswanathan, Homin~K Lee, Kobbi Nissim, Sofya Raskhodnikova,
  and Adam Smith.
\newblock What can we learn privately?
\newblock \emph{SIAM Journal on Computing}, 40\penalty0 (3):\penalty0 793--826,
  2011.

\bibitem[Kulkarni et~al.(2021)Kulkarni, Lee, and Liu]{kll21}
Janardhan Kulkarni, Yin~Tat Lee, and Daogao Liu.
\newblock Private non-smooth empirical risk minimization and stochastic convex
  optimization in subquadratic steps.
\newblock \emph{arXiv preprint arXiv:2103.15352}, 2021.

\bibitem[Ledoux(2006)]{ledoux2006concentration}
Michel Ledoux.
\newblock Concentration of measure and logarithmic sobolev inequalities.
\newblock In \emph{Seminaire de probabilites XXXIII}, pages 120--216. Springer,
  2006.

\bibitem[Li et~al.(2022)Li, Liu, Hashimoto, Inan, Kulkarni, Lee, and
  Thakurta]{LLH+22}
Xuechen Li, Daogao Liu, Tatsunori Hashimoto, Huseyin~A Inan, Janardhan
  Kulkarni, Yin~Tat Lee, and Abhradeep~Guha Thakurta.
\newblock When does differentially private learning not suffer in high
  dimensions?
\newblock \emph{arXiv preprint arXiv:2207.00160}, 2022.

\bibitem[Song et~al.(2021)Song, Steinke, Thakkar, and Thakurta]{sstt21}
Shuang Song, Thomas Steinke, Om~Thakkar, and Abhradeep Thakurta.
\newblock Evading the curse of dimensionality in unconstrained private glms.
\newblock In \emph{International Conference on Artificial Intelligence and
  Statistics}, pages 2638--2646. PMLR, 2021.

\bibitem[Steinke and Ullman(2015)]{steinke2015interactive}
Thomas Steinke and Jonathan Ullman.
\newblock Interactive fingerprinting codes and the hardness of preventing false
  discovery.
\newblock In \emph{Conference on learning theory}, pages 1588--1628. PMLR,
  2015.

\bibitem[Steinke and Ullman(2016)]{su15}
Thomas Steinke and Jonathan Ullman.
\newblock Between pure and approximate differential privacy.
\newblock \emph{Journal of Privacy and Confidentiality}, 7\penalty0
  (2):\penalty0 3--22, 2016.

\bibitem[Talwar et~al.(2015)Talwar, Thakurta, and Zhang]{ttz15}
Kunal Talwar, Abhradeep Thakurta, and Li~Zhang.
\newblock Nearly-optimal private lasso.
\newblock In \emph{Proceedings of the 28th International Conference on Neural
  Information Processing Systems-Volume 2}, pages 3025--3033, 2015.

\bibitem[Tang et~al.(2024)Tang, Panda, Nasr, Mahloujifar, and Mittal]{TPN+24}
Xinyu Tang, Ashwinee Panda, Milad Nasr, Saeed Mahloujifar, and Prateek Mittal.
\newblock Private fine-tuning of large language models with zeroth-order
  optimization.
\newblock \emph{arXiv preprint arXiv:2401.04343}, 2024.

\bibitem[Tardos(2008)]{tar08}
G{\'a}bor Tardos.
\newblock Optimal probabilistic fingerprint codes.
\newblock \emph{Journal of the ACM (JACM)}, 55\penalty0 (2):\penalty0 1--24,
  2008.

\bibitem[Wang et~al.(2017)Wang, Ye, and Xu]{wyx17}
Di~Wang, Minwei Ye, and Jinhui Xu.
\newblock Differentially private empirical risk minimization revisited: Faster
  and more general.
\newblock In \emph{Advances in Neural Information Processing Systems}, pages
  2722--2731, 2017.

\bibitem[Zhang et~al.(2023)Zhang, Thekumparampil, Oh, and He]{ZTOH23}
Liang Zhang, Kiran~Koshy Thekumparampil, Sewoong Oh, and Niao He.
\newblock Dpzero: Dimension-independent and differentially private zeroth-order
  optimization.
\newblock \emph{arXiv preprint arXiv:2310.09639}, 2023.

\bibitem[Zhou et~al.(2020)Zhou, Wu, and Banerjee]{zwb20}
Yingxue Zhou, Zhiwei~Steven Wu, and Arindam Banerjee.
\newblock Bypassing the ambient dimension: Private sgd with gradient subspace
  identification.
\newblock \emph{arXiv preprint arXiv:2007.03813}, 2020.

\end{thebibliography}

\newpage
\appendix
\section{Conclusion and limitation}\label{s5}

In this work, we study dimension-free risk bounds in DP-ERM, offering insights from both an algorithmic advancement perspective and an exploration of fundamental limits. In our first result, we show that under the common unconstrained domain and low-rank gradients assumptions, the regularized exponential mechanism is capable of achieving rank-dependent risk bounds for convex objectives, where the loss can be non-smooth and only zeroth order oracles are given. 

Our second result examines the difference between constrained and unconstrained domain assumptions.
Specifically, we show that without the low-rank gradient assumptions, we achieve the same lower bounds for both the constrained and unconstrained domains. 
In addition, our lower bound is applicable to general $\ell_p$ geometry and has a tighter rate than previous results.

Despite these advancements, several compelling questions remain open in the field. 
First, it is interesting to see if our utility Lemma (Lemma~\ref{lm:sampling_utility}) can be improved, and hence we can tolerate larger $G_k$ for the dimension-independent risk bound.
Second, the current upper bound for $\ell_p$ norms as presented in previous works such as \cite{bgn21,GLLST23} simply adapts the algorithm for $\ell_2$ norms using Hölder's inequality to translate the diameter and Lipschitz constant, leading to a gap between the upper and lower bounds.
Third, our results rely heavily on the convexity assumption on the loss functions, and extending the results to non-convex settings can be meaningful.
Closing the gap is an intriguing open problem. 
Additionally, developing more efficient methods for implementing the exponential mechanism and checking its practical performance are potential avenues for future research.
\section{Preliminary}\label{s2}
We begin with basic definitions.

\begin{definition}[Differential privacy]
A randomized mechanism $\M$ is $(\epsilon,\delta)$-differentially private\footnote{When $\delta>0$, we may refer to it as approximate-DP, and we name the particular case when $\delta=0$ pure-DP sometimes.} if for any event ${\cal O}\in \mathrm{Range}(\M)$ and for any neighboring databases $\D$ and $\D'$ that differ by a single data element, one has
\begin{align*}
    \Pr[\M(\D)\in {\cal O}]\leq \exp(\epsilon)\Pr[\M(\D')\in {\cal O}] +\delta.
\end{align*}
\end{definition}

\begin{definition}[$G$-Lipschitz Continuity]
A function $f:\K\rightarrow \R$ is $G$-Lipschitz continuous with respect to $\ell_p$ geometry if for all $\theta,\theta'\in \K$, one has:
\begin{equation}
    |f(\theta)-f(\theta')|\leq G\|\theta-\theta'\|_p.
\end{equation}
\end{definition}

The following is the classic Pythagorean Theorem.

\begin{lemma}[Pythagorean Theorem for convex set]
\label{lm:Pythagorean}
Letting $\K \subset \R^d$ be a convex set, $y\in \R^d$ and $x=\Pi_{\K} (y)$, then for any $z\in \K$ we have:
\begin{equation}
    \|x-z\|_2\le \|y-z\|_2.
\end{equation}
\end{lemma}
\section{Additional background knowledge}
\subsection{Generalized Linear Model (GLM)}
The generalized linear model (GLM) is a flexible generalization of ordinary linear regression that allows for response variables with error distribution models other than a normal distribution. To be specific,
\begin{definition}[Generalized linear model (GLM)]
\label{def:GLM}
The generalized linear model (GLM) is a special class of ERM problems where the loss function $\ell(\theta,d)$ takes the following inner-product form:
\begin{equation}
    \ell(\theta;z)=\ell(\langle \theta,x \rangle;y)
\end{equation}
for $z=(x,y)$. Here, $x\in \R^d$ is usually called the feature vector and $y\in \R$ is called the response.
\end{definition}

We also outline some basic properties of differential privacy, which will be used in our lower bounds (see \cite{dwork2014algorithmic} for proof details). 
\begin{proposition}[Group privacy]
\label{prop:group_privacy}
If $\M: X^n\to Y$ is $(\epsilon,\delta)$-differentially private mechanism, then for all pairs of datasets $x,x'\in X^n$, then $\M(x), \M(x')$ are $(k\epsilon, k \delta e^{k\epsilon})$-indistinguishable when $x,x'$ differs on at most $k$ locations.
\end{proposition}

\begin{proposition}[Post processing]
If $\M: X^n\to Y$ is $(\epsilon,\delta)$-differentially private and $\A: Y\to Z$ is any randomized function, then $\A\circ \M: X^n\to Z$ is also $(\epsilon,\delta)$-differentially private.
\end{proposition}


\subsection{Construction of fingerprinting codes}
To address the digital watermarking problem, Fingerprinting codes were introduced by \cite{bs98}. Imagine a company selling software to users. A fingerprinting code is a pair of randomized algorithms $(\Gen, \Trace)$, where $\Gen$ generates a length $d$ code for each user $i$. To prevent any malicious coalition of users copy and distributing the software, the $\Trace$ algorithm can trace one of the malicious users, given a code produced by the coalition of users. They may only can the bits with a divergence in the code: any bit in common is potentially vital to the software and risky to change.  

In this section, we introduce the fingerprinting code used by \cite{buv18}, which is based on the first optimal fingerprinting code \cite{tar08} with additional error robustness. The mechanism of the fingerprinting code is described in Algorithm \ref{alg1} for completeness.

\begin{algorithm}[tb]
\caption{The Fingerprinting Code $(\Gen,\Trace)$} 
\begin{algorithmic}
\STATE \textbf{Sub-procedure $\Gen'$}:
\STATE Let $d=100n^2 \log(n/\xi)$ be the length of the code.
\STATE Let $t=1/300n$ be a parameter and let $t'$ be such that $sin^2t'=t$.
\FOR{$j=1,...,d$:}
\STATE Choose random $r$ uniformly from $[t',\pi/2-t']$ and let $p_j=sin^2r_j$. Note that $p_j\in [t,1-t]$.
\STATE For each $i=1,...,n$, set $C_{ij}=1$ with probability $p_j$ independently.
\ENDFOR
\STATE {\bf Return:} $C$
\STATE

\STATE \textbf{Sub-procedure $\Trace'(C,c')$}:
\STATE Let $Z=20n \log(n/\xi)$ be a parameter.
\STATE For each $j=1,...,d$, let $q_j=\sqrt{(1-p_j)/p_j}$.
\STATE For each $j=1,...,d$, and each $i=1,...,n$, let $U_{ij}=q_j$ if $C_{ij}=1$ and $U_{ij}=-1/q_j$ else wise.
\FOR{each $i=1,...,n$:}
\STATE Let $S_i(c')=\sum_{j=1}^d c_j' U_{ij}$
\STATE Output $i$ if $S_i(c')\ge Z/2$.
\STATE Output $\perp$ if $S_i(c')< Z/2$ for every $i=1,...,n$.
\ENDFOR
\STATE

\STATE \textbf{Main-procedure $\Gen$}:
\STATE Let $C$ be the (random) output of $\Gen'$, $C\in \{0,1\}^{n\times d}$
\STATE Append $2d$ 0-marked columns and $2d$ 1-marked columns to $C$.
\STATE Apply a random permutation $\pi$ to the columns of the augmented codebook.
\STATE Let the new codebook be $C'\in \{0,1\} ^{n\times 5d}$.
\STATE {\bf Return:} $C'$.
\STATE

\STATE \textbf{Main-procedure $\Trace(C,c')$}:
\STATE Obtain $C'$ from the shared state with $\Gen$.
\STATE Obtain $C$ by applying $\pi^{-1}$ to the columns of $C'$ and removing the dummy columns.
\STATE Obtain $c$ by applying $\pi^{-1}$ to $c'$ and removing the symbols corresponding to fake columns.
\STATE {\bf Return:} $i$ randomly from $\Trace'(C,c)$.
\end{algorithmic} 
\label{alg1}
\end{algorithm}

The sub-procedure part is the original fingerprinting code in \cite{tar08}, with a pair of randomized algorithms $(\Gen,\Trace)$. The code generator $\Gen$ outputs a codebook $C\in \{0,1\}^{n\times d}$. The $ith$ row of $C$ is the codeword of user $i$. The parameter $d$ is called the length of the fingerprinting code.

We make the formal definition of fingerprinting codes:
\begin{definition}[fingerprinting codes]
Given $n,d\in\N,\xi \in(0,1]$, a pair of (random) algorithms $(\Gen,\Trace)$ is called an $(n,d)$-fingerprinting code with security $\xi\in(0,1]$ if $\Gen$ outputs a code-book $C\in\{0,1\}^{n\times d}$ and for any (possibly randomized) adversary $\A_{FP}$ and any subset $S\subseteq[n]$, if we set $c\leftarrow_R\A_{FP}(C_S)$, then
\begin{itemize}
    \item $\Pr[c\in F(C_S)\bigwedge \Trace(C,c)=\perp]\leq \xi$ 
    \item $\operatorname{Pr}\left[\Trace \left(C, c\right) \in[n] \backslash S\right] \leq \xi$
\end{itemize}
where $F\left(C_{S}\right)=\left\{c\in\{0,1\}^{d} \mid \forall j \in[d], \exists i \in S, c_{j}=c_{i j}\right\}$, and the probability is taken over the coins of $\Gen,\Trace$ and $\A_{FP}$.
\end{definition}

Fingerprint codes imply the hardness of privately estimating the mean of a dataset over $\{0,1\}^d$. Otherwise, the coalition of users can simply use the rounded mean of their codes to produce the copy. Then the DP-ERM problem can be reduced to privately estimating the mean by using the linear loss whose minimizer is precisely the mean.

The security property of fingerprinting codes asserts that any codeword can be “traced” to a user $i$. Moreover, we require that the fingerprinting code can find one of the malicious users even when they get together and combine their codewords in any way that respects the marking condition. That is, a tracing algorithm $\Trace$ takes as inputs the codebook $C$ and the combined codeword $c'$ and outputs one of the malicious users with high probability.

The sub-procedure $\Gen'$ first uses a $\sin^2 x$ like distribution to generate a parameter $p_j$ (the mean) for each column $j$ independently, then generates $C$ randomly by setting each element to be 1 with probability $p_j$ according to its location. The sub-procedure $\Trace'$ computes a threshold value $Z$ and a 'score function' $S_i(c')$ for each user $i$, then reports $i$ when its score is higher than the threshold. 

The main procedure was introduced in \cite{buv18}, where $\Gen$ adds dummy columns to the original fingerprinting code and applies a random permutation. $\Trace$ can first 'undo' the permutation and remove the dummy columns, then use $\Trace'$ as a black box. This procedure makes the fingerprinting code more robust in tolerating a small fraction of errors to the marking condition. 

In particular, they prove the fingerprinting code Algorithm \ref{alg1} has the following property.

\begin{theorem}[Theorem 3.4 in \cite{buv18}]
\label{thm:construction_fc}
For every $d$, and $\gamma \in (0, 1]$, there exists a  $(n,d)$-fingerprinting code with security $\gamma$ robust to a 1/75 fraction of errors for, for 
$$
n=\Omega(\sqrt{d/\log(1/\gamma)})
$$
\end{theorem}

\section{Example for Pure-DP}
In the construction of lower bounds for constrained DP-ERM in \cite{bst14}, they chose the linear function $\ell(\theta;z)=\langle \theta,z\rangle$ as the objective function, which is not applicable in the unconstrained setting because it could decrease to negative infinity. Instead, we extend the linear loss in unit $\ell_2$ ball to the whole $\R^d$ while preserving its Lipschitzness and convexity.
We use such an extension to define our loss function in the unconstrained case. 
Namely, we define
\begin{equation}
\label{eq:func_for_pureDP}
    \ell(\theta;z)=\min_{\|y\|_2\le 1} -\langle y,z \rangle+\|\theta-y\|_2
\end{equation}
for all $\theta,z$ in the unit $\ell_2$ ball,
which is convex, 1-Lipschitz and equal to $-\langle \theta,z \rangle$ when $\|\theta\|_2\le 1$ according to Lemma \ref{lem4}. 
Specifically, it's easy to verify that $\ell(\theta;0)=\max\{0, \|\theta\|_2-1\}$. When $\|z\|_2=1$, one has
\begin{equation}
    \ell(\theta;z)\ge \min_{\|y\|_2\le 1} -\langle y,z \rangle\ge -1,
\end{equation}
where the equation holds if and only if $\theta=z$.

For any dataset $\D=\{z_1,...,z_n\}$, we define $ L(\theta;\D)=\frac{1}{n}\sum_{i=1}^n \ell(\theta;z_i).$ 
We need the following lemma from \cite{bst14} to prove the lower bound.
The proof is similar to that of Lemma 5.1 in \cite{bst14}, except that we change the construction by adding points $\mathbf{0}$ (the all-zero $d$ dimensional vector) as our dummy points. For completeness, we include it here.

\begin{restatable}[Part-One of Lemma 5.1 in \cite{bst14} with slight modifications]{lemma}{reductionone}
\label{lem5}
Let $n,d\ge 2$ and $\epsilon>0$. There is a number $n^*=\Omega(\min(n,\frac{d}{\epsilon}))$ such that for any $\epsilon$-differentially private algorithm $\A$, there is a dataset $\D=\{z_1,...,z_n\}\subset \{\frac{1}{\sqrt{d}},-\frac{1}{\sqrt{d}}\}^d\cup \{ \mathbf{0}\}$ with $\|\sum_{i=1}^n z_i\|_2=n^*$ such that, with probability at least $1/2$ (taken over the algorithm random coins), we have 
\begin{equation}
    \|\A(\D)-q(\D)\|_2=\Omega(\min(1,\frac{d}{n\epsilon})),
\end{equation}
where $q(\D)=\frac{1}{n}\sum_{i=1}^n z_i$.
\end{restatable}

Lemma \ref{lem5} basically says that for any $\epsilon$-DP algorithm, it's impossible to for it to estimate the average of some dataset $z_1,...,z_n$ with accuracy $o(\min(1,\frac{d}{n\epsilon}))$. 
Using the loss functions defined in Equation~(\ref{eq:func_for_pureDP}), Lemma~\ref{lem5} and our reduction theorem~\ref{thm4}, we have the following theorem,
whose proof can be found in the appendix.

\begin{restatable}[Lower bound for $\epsilon$-differentially private algorithms]{theorem}{reductiontwo}
\label{thm1}
Let $n,d$ be large enough and $\epsilon>0$. For every $\epsilon$-differentially private algorithm with output $\theta^{priv}\in \R^d$, there is a dataset $\D=\{z_1,...,z_n\}\subset \{\frac{1}{\sqrt{d}},-\frac{1}{\sqrt{d}}\}^d \cup \{ \bzero\}$ such that, with probability at least $1/2$ (over the algorithm random coins), we must have that
\begin{equation}
    L(\theta^{priv};\D)-\min_{\theta\in\R^d}L(\theta;\D)=\Omega(\min(1,\frac{d}{n\epsilon})).
\end{equation}
\end{restatable}

As mentioned before, this lower bound suggests the necessity of additional assumptions for dimension-independent results in pure DP.
\section{Omitted proof for Section \ref{sec sam}}
\label{appendix:omitted_proof_upper_bound}
The norm $\|\cdot\|$ means the $\ell_2$ norm for simplicity in this section.

\errbyvar*
\begin{proof}
The proof involves studying the following quantity
\begin{align*}
    V_t:=\E_{x\sim\pi_t}f(x).
\end{align*}
For simplicity, we let $\phi(x):=\frac{\mu}{2}\|x\|^2$ to be the regularized term and have
\begin{align*}
    \frac{\d}{\d t}V_t=& ~ \frac{\d}{\d t}\frac{\int f(x)e^{-tf(x)-\phi(x)}\d x}{\int e^{-tf(x)-\phi(x)}\d x}\\
    =&~ \frac{-\int f^2(x)e^{-tf(x)-\phi(x)}\d x}{\int e^{-tf(x)-\phi(x)}\d x}+\left(\frac{\int f(x)e^{-tf(x)-\phi(x)}\d x}{\int e^{-tf(x)-\phi(x)}\d x}\right)^2\\
    =& ~ (\E_{x\sim\pi_t}f(x))^2-\E_{x\sim \pi_t}f^2(x) = -\Var_{x\sim \pi_t}(f).
\end{align*}
Hence we have
$\E_{x\sim\pi}f(x)= V_1
    = V_{\infty}-\int_{1}^{\infty}\frac{\d}{\d t}V_t\d t
    = f(x^*)+\int_{1}^{\infty}\Var_{x\sim\pi_t}(f)\d t$.
\end{proof}

\sampleone*

\begin{proof}
Note that
\begin{align*}
    \Var_{x\sim\pi}f(x)=&~ \Var_{x\sim\pi}(f(x)+\frac{\mu}{2}\|x\|^2-\frac{\mu}{2}\|x\|^2)\\
    =& ~ \Var_{x\sim\pi}(f(x)+\frac{\mu}{2}\|x\|^2)+\Var_{x\sim\pi}(\frac{\mu}{2}\|x\|^2)-2\Cov_{x\sim\pi}\big((f(x)+\frac{\mu}{2}\|x\|^2)(\frac{\mu}{2}\|x\|^2)\big)\\
    \leq & ~ \Var_{x\sim\pi}(f(x)+\frac{\mu}{2}\|x\|^2)+\Var_{x\sim\pi}(\frac{\mu}{2}\|x\|^2)+2\sqrt{\Var_{x\sim\pi}(f(x)+\frac{\mu}{2}\|x\|^2)\cdot\Var_{x\sim\pi}(\frac{\mu}{2}\|x\|^2)}\\
    \leq & ~ 2\Var_{x\sim\pi}(f(x)+\frac{\mu}{2}\|x\|^2)+2\Var_{x\sim\pi}(\frac{\mu}{2}\|x\|^2)\\
    \leq & ~ 2d+\frac{\mu^2}{2}\Var_{x\sim\pi}(\|x\|^2),
\end{align*}
where the last line follows from Lemma~\ref{lm:var_by_dim}.
It suffices to consider $\Var_{x\sim\pi}(\|x\|^2)$, for which we have
\begin{align*}
    \Var_{x\sim\pi}(\|x\|^2)=& ~ \E_{x\sim\pi}\left(\|x\|^2-\E_{x\sim\pi}\|x\|^2\right)^2\\
    =& ~ \E_{x\sim\pi}(\|x\|^2-\E_{x\sim\pi}\|x-\ox\|^2-\|\ox\|^2)^2\\
    \leq & ~ 2\E_{x\sim\pi}(\|x-\ox\|^2-\E_{x\sim\pi}\|x-\ox\|^2)^2+8\E_{x\sim\pi}(\ox^\top(x-\ox))^2\\
    =& ~ 2\Var_{x\sim\pi}\|x-\ox\|^2+8\E_{x\sim\pi}(\ox^\top(x-\ox))^2.
\end{align*}


Since $\pi$ is $\mu$-strongly log-concave,
we have
\begin{align*}
    \Var_{x\sim\pi}\|x-\ox\|^2\leq \frac{1}{\mu}\E_{x\sim\pi}\|2(x-\ox)\|^2\leq \frac{4}{\mu}\tr
    \Cov(\pi),
\end{align*}
where the first inequality follows from Brascamp–Lieb inequality.
As $\pi$ is $\mu$-strongly log-concave, we know $\Cov(\pi)\preceq \frac{1}{\mu}\cdot I$ and hence we know $\Var_{x\sim\pi}\|x-\ox\|^2\leq 4d/\mu^2$. 
Similarly, we have
\begin{align*}
    \E_{x\sim\pi}(\ox^\top(x-\ox))^2=\ox^\top\Cov(\pi)\ox\leq \frac{1}{\mu}\cdot\|\ox\|^2.
\end{align*}

Combining these, we have
\begin{align*}
    \Var_{x\sim \pi}f(x)\leq 2d+\frac{\mu^2}{2}(4d/\mu^2+\|\ox\|^2/\mu).
\end{align*}
\end{proof}

\sampletwo*
\begin{proof}
For simplicity, denote $Qx$ by $x_1$.
Similar to $x_1$, we may write $x^*_1$ for $Qx^*$.
For simplicity, we denote $h(x)=f(x)+\frac{\mu}{2}\|x\|^2$.
We prove this lemma by considering the Langevin diffusion associated with $\pi$, that is 
\begin{align*}
    \d Y_t=-\nabla h(Y_t)\d t + \sqrt{2}\d B_t,
\end{align*}
where $B_t$ is a $d$-dimensional Brownian motion and we denote its associated semi-group $(P_t)_{t\ge 0}$.

Consider the function $g(x):=\|Q(x-x^*)\|_2^2=\|x_1-x^*_1\|_2^2$.
Recall the infinitesimal generator $A$ such that
\begin{align*}
    Ag(x)=-\langle \nabla h(x),\nabla g(x)\rangle+\Delta g(x).
\end{align*}

Recall that $\nabla h(x^*)=0$ as $x^*$ is the global optimum, and by the strong convexity of $h$, we have
\begin{align*}
    Ag(x)=&~ -2\langle\nabla h(x)-\nabla h(x^*),Q(x-x^*)\rangle+2k\\
    \leq & -2\mu \|Q(x-x^*)\|_2^2+2k\\
    =& ~ -2\mu g(x)+2k.
\end{align*}

For any $t\geq 0$ and $x\in\R^d$, we let $v(t,x)=P_tg(x)$. 
We have $\frac{\partial v(t,x)}{\partial t}=P_tAg(x)$ and hence
\begin{align*}
    \frac{\partial v(t,x)}{\partial t}= P_tAg(x)\leq -2\mu P_tg(x)+2k=& -2\mu v(t,x)+2k.
\end{align*}

By Grönwall's inequality, we know for all $t\ge0$ and $x\in\R^d$, one has
\begin{align*}
    \E[\|Q(Y_t-x^*)\|_2^2]\leq \|Q(x-x^*)\|_2^2e^{-2\mu t}+\frac{k}{\mu}(1-e^{-2\mu t}).
\end{align*}

Then for any $c>0$ and $t>0$, we know
\begin{align*}
    \E_{x\sim\pi}(g\wedge c):=&~\pi(g\wedge c)=\pi P_t(g\wedge c)\leq \pi(P_t g\wedge c)\\
    =&~\int \pi(\d x)c\wedge \{\|Q(x-y)\|^2e^{-2\mu t}+\frac{k}{\mu}(1-e^{-2\mu t}) \}\\
    \leq &~ \pi(c\wedge e^{-2\mu t}g)+(1-e^{-2\mu t})k/\mu\\
    =& ~ \E_{x\sim \pi }(c\wedge ge^{-2\mu t})+(1-e^{-2\mu t})k/\mu.
\end{align*}
Hence we know $\E_{x\sim \pi}(g)\leq k/\mu$.
\end{proof}

\samplethree*
\begin{proof}
For simplicity, let $x_1=Qx$ and $x_2=(I-Q)x$. Without loss of generality, assume $x_1$ is the first $k$ coordinates of $x$, and hence $\|\frac{\partial f}{\partial x_2}\|\le G_k$.
We say $x_2\sim\pi$ if its density is proportional to $\frac{\int e^{-f(x_1,x_2)-\frac{\mu}{2}(\|x_1\|^2+\|x_2\|^2)}\d x_1}{\int \int e^{-f(x_1,x_2)-\frac{\mu}{2}(\|x_1\|^2+\|x_2\|^2)}\d x_1\d x_2}$, and as for the distribution of $x_2$ conditional on $x_1$, we denote it by $x_2\mid x_1 \sim\pi$, whose density is $\frac{ e^{-f(x_1,x_2)-\frac{\mu}{2}\|x_2\|^2}}{\int e^{-f(x_1,x_2)-\frac{\mu}{2}\|x_2\|^2}\d x_2}$.
And the meanings for $x_1\sim\pi$ and $x_1\mid x_2\sim\pi$ follow similarly.

By the variance decomposition, we have
\begin{align}
\label{eq:var_decomp}
    \Var_{(x_1,x_2)\sim\pi}f(x)=\E_{x_2\sim\pi}\Var_{x_1\mid x_2\sim\pi}f(x)+\Var_{x_2\sim\pi}(\E_{x_1\mid x_2\sim\pi}f(x)).
\end{align}
For simplicity, we may hide  ``$\sim\pi$'' in the subscripts.
It suffices to bound the two terms in Equations~\eqref{eq:var_decomp} separately.
For the first term, since we are considering the variance conditional on $x_2$, by Lemma~\ref{lm:regular_var_by_dim} we have
\begin{align*}
    \Var_{x_1\mid x_2}f(x)\leq 4k+\frac{\mu}{2}\|\E_{x_1\mid x_2}x_1\|^2.
\end{align*}
Hence we have
\begin{align*}
    \E_{x_2\sim\pi}\Var_{x_1\mid x_2\sim\pi}f(x)\leq &~ 4k+\frac{\mu}{2}\E_{x_2}\|\E_{x_1\mid x_2}x_1\|^2\\
    \leq & ~ 4k+\frac{\mu}{2}\E_{x}\|x_1\|^2,
\end{align*}
where the last line follows from Law of total expectation and Jensen's Inequality.
Again, since $\pi$ is $\mu$-strongly log-concave, we have $\Cov(\pi)\preceq \frac{1}{\mu}\cdot I$
and hence $\E\|x_{1}-\E x_{1}\|^{2}\leq k/\mu $. Therefore, we have
\[
\E_{x_{2}}\Var_{x_{1}|x_{2}}f\leq\frac{9}{2}k+\frac{\mu}{2}\cdot\|\E x_{1}\|^{2}.
\]
Let $y=\arg\min_{x}f(x)+\frac{\mu}{2}\|x\|^{2}$.
By Lemma~\ref{lm:bound_first_k_dim}, we can show $\|\E_{x} x_{1}-y_{1}\|_{2}=O(\sqrt{k/\mu})$.
Hence we have
\begin{align*}
    \E_{x_2}\Var_{x_1\mid x_2}f(x)\lesssim k+\mu\cdot\|y_1\|^2.
\end{align*}
Noting that $f(y)+\frac{\mu}{2}\|y\|^2\leq f(x^*)+\frac{\mu}{2}\|x^*\|^2$ and $f(y)\geq f(x^*)$, we have $\|y_1\|^2\leq \|y\|^2\leq \|x^*\|^2$ and hence
\begin{align}
\label{eq:term_one}
    \E_{x_2}\Var_{x_1\mid x_2}f(x)\lesssim k+\mu\cdot\|x^*\|^2.
\end{align}


Now we bound the second term in Equation~\eqref{eq:var_decomp}.
For simplicity, we use $\phi(x)=\frac{\mu}{2}\|x\|^2$ and denote
\begin{align*}
    g(x_2):=\E_{x_1\mid x_2\sim\pi}f(x).
\end{align*}

We use $\partial_2$ for taking partial derivative with respect to $x_2$, and one has
\begin{align*}
    \partial_2 g(x_2)=& ~\partial_2\frac{\int f(x_1,x_2)\exp(-f(x_1,x_2)-\phi(x_1,x_2))\d x_1}{\int \exp(-f(x_1,x_2)-\phi(x_1,x_2))\d x_1}\\
    =& ~ \frac{\int (\partial_2 f) \exp(-f-\phi)\d x_1}{\int \exp(-f-\phi)\d x_1}  -\frac{\int f\cdot\partial_{2}(f+\phi)\cdot\exp(-f-\phi)dx_{1}}{\int\exp(-f-\phi)dx_{1}}\\
    & ~ +\frac{\int f\cdot\exp(-f-\phi)dx_{1}\cdot\int\partial_{2}(f+\phi)\exp(-f-\phi)dx_{1}}{(\int\exp(-f-\phi)dx_{1})^{2}}\\
    = &~ \E_{x_{1}}\partial_{2}f+(\E_{x_{1}}f)(\E_{x_{1}}\partial_{2}(f+\phi))-(\E_{x_{1}}(f\partial_{2}(f+\phi)))\\
    = &~ \E_{x_{1}}\partial_{2}f-\E_{x_{1}}((f-\E_{x_{1}}f)\partial_{2}(f+\phi)))\\
    = & ~ \E_{x_{1}}\partial_{2}f-\E_{x_{1}}((f-\E_{x_{1}}f)\partial_{2}f),
\end{align*}
where the last equality follows from that $\partial_2\phi=\mu x_2$.


By Brascamp–Lieb inequality, we get 
\begin{align*}
    \Var_{x_2}(\E_{x_1\mid x_2}f)=&~\Var_{x_2}(g)\\
    \lesssim &~\frac{1}{\mu}\cdot\E_{x_2}\|\partial_2g\|^2\\
    \lesssim & ~\frac{1}{\mu}( \E_{x}\|\partial_2 f\|^2+\E_{x_2}\Big(\Var_{x_1\mid x_2}f\cdot \E_{x_1}\|\partial_2 f\|^2\Big))\\
    \lesssim & ~ \frac{1}{\mu}( G_k^2+G_k^2\cdot \E_{x_2}\Var_{x_1\mid x_2}f)\\
    \lesssim & ~\frac{G_k^2}{\mu} (k+\mu\|x^*\|^2),
\end{align*}
where the last line follows from Equation~\eqref{eq:term_one}.
\end{proof}

\samplefour*
\begin{proof}
Let $p_t(x)\propto \exp(-\eta tf(x)-\frac{\eta\mu}{2}\|x\|^2)$. 
By Lemma~\ref{lm:err_by_var}, we know
\begin{align*}
    \E_{x\sim p}\eta f(x)=\min_x\eta f(x)+\int_1^{\infty}\Var_{x\sim p_t}\eta f(x)\d t.
\end{align*}
By Lemma~\ref{lm:var_k}, we have
\begin{align*}
    \Var_{x\sim p_t}\eta f(x)=&~\frac{1}{t^2}\Var_{x\sim p_t}\eta tf(x)\\
    \lesssim & ~ \frac{1}{t^2}(k+\eta\mu\cdot\|x^*\|^2)(\frac{t^2\eta^2G_k^2}{\eta\mu}+1).\\
\end{align*}
Hence we get
\begin{align*}
    \E_{x\sim p}\eta f(x)-\min_{x}\eta f(x)\lesssim \eta\mu\|x^*\|^2+\int_1^{\infty}\min_k\{\frac{\eta G_k^2}{\mu}(k+\eta\mu\cdot\|x^*\|^2)+\frac{k}{ t^2}\}\d t.
\end{align*}
and hence
\begin{align*}
     \E_{x\sim p}f(x)-\min_{x}f(x)\lesssim & 
    \mu\|x^*\|^2+ \int_1^{\infty}\min_k\{\frac{G_k^2}{\mu}(k+\eta\mu\cdot\|x^*\|^2)+\frac{k}{\eta t^2}\}\d t.
\end{align*}
\end{proof}

\subsection{Proof of Theorem \ref{thm main app}} 
\paragraph{Privacy Guarantee:}
We first introduce the following lemma on the GDP of exponential mechanism.
\begin{lemma}[GDP of regularized exponential mechanism \cite{gll22}]
\label{thm:GLL22}
    Given convex set $\cK\subseteq \R^d$ and $\mu$-strongly convex functions $F,\Tilde{F}$ over $\cK$. Let $P,Q$ be distributions over $\cK$ such that $P(x)\propto e^{-F(x)}$ and $Q(x)\propto e^{-\Tilde{F}(x)}$.
If $\Tilde{F}-F$ is $G$-Lipschitz over $\cK$, then for all $\epsilon>0$,
\begin{align*}
    \delta\big(P\|Q\big)(\epsilon) 
    \leq \delta\big(\cN( 0,1)\|\cN(\frac{G}{\sqrt{\mu}},1)\big)(\epsilon).
\end{align*}
\end{lemma}
The privacy cure between two random variables $X$ and $Y$ is defined as:
\begin{align*}
    \delta(X\|Y)(\epsilon):=\sup_{S}\Pr[Y\in S]-e^{\epsilon}\Pr[X\in S].
\end{align*}
One can explicitly calculate the privacy curve of a Gaussian mechanism as 
\begin{align*}
    \delta(\cN(0,1)\|\cN(s,1))(\epsilon)=\Phi(-\frac{\epsilon}{2}+\frac{s}{2})-e^\epsilon\Phi(-\frac{\epsilon}{s}-\frac{s}{2}),
\end{align*}
where $\Phi$ is the Gaussian cumulative distribution function (CDF).
Then the privacy guarantee follows immediately from Lemma~\ref{thm:GLL22} by our parameter settings. 

\paragraph{Utility Guarantee:}
As for the utility guarantee, by Lemma~\ref{lm:sampling_utility}, we have 
\begin{align*}
    \E[L(\theta^{app};\D)-L(\theta^*;\D)]
    \lesssim  &~\mu\|x^*\|^2+\int_1^{d}\min_k\{\frac{G_k^2}{\mu}(k+\eta\mu\cdot\|x^*\|^2)+\frac{k}{\eta t^2}\}\d t +\int_{d}^{\infty}\frac{d}{\eta t^2}\d t\\
    \lesssim&~ \frac{GC\sqrt{k\log(1/\delta)}}{n\epsilon}+\frac{k}{\eta}+\frac{G_k^2}{\mu}(k+\eta \mu \|x^*\|^2)d.
\end{align*}
When $G_k\le \frac{G}{n\epsilon\sqrt{d}}$ as in the precondition, we get the desired utility guarantee.

\paragraph{Oracle Complexity:}
We make use of the following sampler:
\begin{lemma}[\cite{gll22}]
\label{thm:sampler}
Given a convex set $\K\subset\R^d$ of diameter $C$, a $\mu$-strongly convex functions $\psi$ and a family of $G$-Lipschitz convex functions $\{f_i\}_{i\in I}$ defined over $\K$. Define the function $F(x):=\E_{i\in I}f_i(x)+\psi(x)$.
For any $0<\delta<1/2$,
one can generate a random point $x$ whose distribution has $\delta$ total variation distance to the distribution proportional to $\exp(-F)$ in 
\begin{align*}
    T:=\Theta\left(\frac{G^2}{\mu} \log ^2\left(\frac{G^2\left(d / \mu+C^2\right)}{\delta}\right)\right) \text{ steps},
\end{align*}
where each step accesses only $O(1)$ values of $f_i$ and samples from $\exp(-\psi(x)-\frac{1}{2\lambda}\|x-y\|^2)$ for $O(1)$ many $y$ with $\lambda=\Theta(G^{-2}/\log(T/\delta))$.
\end{lemma}

This sampler only works in a bounded domain.
To apply for this sampler, we need the following concentration result:
\begin{lemma}[Gaussian Concentration,\cite{ledoux2006concentration}]
    Let $X\sim\exp(-f)$ for $1/\eta$-strongly convex function and $g$ is $G$-Lipshcitz, then
    \begin{align*}
        \Pr[g(X)-\E g(X)\ge t]\le e^{-t^2/(2\eta G)}.
    \end{align*}
\end{lemma}

Define $\pi$ to be the density proportional to $\exp(-\eta(L(\theta;D)+\mu\|\theta\|^2/2))$.
Define $g(\theta):=\|\theta\|$, and we know $\E_{\theta}\|\theta-\theta_\mu^*\|^2\le d/\mu$ by the standard analysis in sampling, where $\theta_\mu^*:=\arg\min L(\theta;D)+\mu\|\theta\|^2/2$.
By Assumption~\ref{assump:unbounded}, we know $\|\theta_\mu^*\|\le C$.
Hence we should restrict $\pi$ in a ball of radius $O(\sqrt{d/\mu}+\sqrt{G\log(4/\delta)}/\eta \mu)$ and get $\pi'$, the TV distance between $\pi$ and $\pi'$ is at most $\delta/4$.

Directly applying Lemma~\ref{thm:sampler} and the parameter setting in Theorem~\ref{thm main app} with $T=O(\frac{\eta G^2}{\mu}\log^2(dn/\delta))$, constructing the sample $x^{app}$ from $\pi'$ requires only $\Tilde{O}(n^2\epsilon^2)$ steps and zero-th order queries in expectation, such that the TV distance between our output $x^{app}$ and the objective distribution $\pi'$ is at most $\delta/4$.
Then the TV distance between the distribution $x^{app}$
and $\pi$ is at most $\delta/2$ by triangle inequality.

\newpage
\section{Omitted proof for Section~\ref{sec:reduction} }

\subsection{Proof of Theorem \ref{thm4}}
\reductionthm*
\begin{proof}
Without loss of generality,
let $\K=\{\theta:\|\theta-\theta_0\|_2\leq C\}$ be the $\ell_2$ ball around $\theta_0$,
let $\ell(\theta;z)$ be the convex functions used in Definition~\ref{prop1}, and as mentioned we can find our loss functions $\tilde{\ell}(\theta;z)=\min_{y\in \K} \ell(y;z) +G\|\theta-y\|_2$.
As $\theta^*\in \K$, we know that 
\begin{align}
\label{eq:same_minimum}
\vspace{-2mm}
    \tilde{L}(\theta^*;\D)=\min_{\theta\in\K} L(\theta;D).
\end{align}
Denote $\tilde{\theta}^{priv}=\Pi_{\K} (\theta^{priv})$ the projected point of $\theta^{priv}$ to $\K$. 
Because post-processing keeps privacy, outputting $\tilde{\theta}^{priv}$ is also $(\epsilon,\delta)$-DP.
By Definition~\ref{prop1}, we have
\begin{align}
\label{eq:lowerbound_constrained}
\vspace{-2mm}
    L(\tilde{\theta}^{priv};\D)-\min_{\theta}L(\theta;\D)=\Omega(f(d,n,\epsilon,\delta,G,C)).
\end{align}
If $\tilde{\theta}^{priv}=\theta^{priv}$, which means $\theta^{priv}\in \K$, then because $\tilde{\ell}(\theta;z)$ is equal to $\ell(\theta;z)$ for any $\theta\in \K$ and $z$, one has $\tilde{L}(\theta^{priv};\D)=\tilde{L}(\tilde{\theta}^{priv};\D)=L(\tilde{\theta}^{priv};\D)$.

If $\tilde{\theta}^{priv}\ne \theta^{priv}$ which means $\theta^{priv} \notin \K$, then since $\ell(\cdot;z)$ is $G$-Lipschitz, for any $z$, we have that (denoting $y^*=\arg\min_{y\in \K} \ell(y;z) +G\|\theta^{priv}-y\|_2$):
\begin{align*}
    \tilde{\ell}(\theta^{priv};z)&=\min_{y\in \K} \ell(y;z) +G\|\theta^{priv}-y\|_2\\
    &=\ell(y^*;z) +G\|\theta^{priv}-y^*\|_2\\
    &\ge \ell(y^*;z) +G\|\tilde{\theta}^{priv}-y^*\|_2\\
    &\ge \min_{y\in \K} \ell(y;z) +G\|\tilde{\theta}^{priv}-y\|_2\\
    &=\tilde{\ell}(\tilde{\theta}^{priv};z),
\end{align*}
where the third line is by the Pythagorean Theorem for the convex set, see Lemma~\ref{lm:Pythagorean}.
We have $\tilde{L}(\theta^{priv};\D)\ge\tilde{L}(\tilde{\theta}^{priv};\D)=L(\tilde{\theta}^{priv};\D)$.
In a word, we get
\begin{align}
    \label{eq:projection_decrease_value}
    \tilde{L}(\theta^{priv};\D)\ge\tilde{L}(\tilde{\theta}^{priv};\D)=L(\tilde{\theta}^{priv};\D).
\end{align}

Combining Equation~(\ref{eq:same_minimum}), (\ref{eq:lowerbound_constrained}) and (\ref{eq:projection_decrease_value}) together, we have that
\begin{align*}
    &~\tilde{L}(\theta^{priv};\D)-\tilde{L}(\theta^*;\D)\\
    = &~ \tilde{L}(\theta^{priv};\D)-\min_{\theta}L(\theta;\D)\\
    \geq & ~ L(\tilde{\theta}^{priv};\D)-\min_{\theta}L(\theta;\D)\\
    \geq & ~ \Omega(f(d,n,\epsilon,\delta,G,C)).
\end{align*}
\end{proof}

\subsection{Proof of Lemma \ref{lem5}}
\reductionone*
\begin{proof}

By using a standard packing argument we can construct $K=2^{\frac{d}{2}}$ points $z^{(1)},...,z^{(K)}$ in $\{\frac{1}{\sqrt{d}},-\frac{1}{\sqrt{d}}\}^d \cup \{\bzero\}$ such that for every distinct pair $z^{(i)},z^{(j)}$ of these points, we have
\begin{equation}
    \|z^{(i)}-z^{(j)}\|_2\ge\frac{1}{8}
\end{equation}

It is easy to show the existence of such a set of points using the probabilistic method (for example, the Gilbert-Varshamov construction of a linear random binary code). 

Fix $\epsilon>0$ and define $n^{\star}=\frac{d}{20\epsilon}$. Let’s first consider the case where $n\le n^{\star}$. We construct $K$ datasets $\D^{(1)},...,\D^{(K)}$ where for each $i\in [K]$, $\D^{(i)}$ contains $n$ copies of $z^{(i)}$. Note that $q(\D^{(i)})=z^{(i)}$, we have that for all $i\ne j$,
\begin{equation}
    \|q(\D^{(i)})-q(\D^{(j)})\|_2\ge\frac{1}{8}
\end{equation}

Let $\A$ be any $\epsilon$-differentially private algorithm.
Suppose that for every $\D^{(i)}, i\in[K]$, with probability at least $1/2$, $\|\A(\D^{(i)})-q(\D^{(i)})\|_2 < \frac{1}{16}$,i.e.,$Pr[\A(\D^{(i)})\in B(\D^{(i)})]\ge \frac{1}{2}$ where for any dataset $\D$, $B(\D)$ is defined as
\begin{equation}
    B(\D)=\{x\in R^d:\|x-q(\D)\|_2<\frac{1}{16} \}
\end{equation}

Note that for all $i\ne j$, $\D^{(i)}$ and $\D^{(j)}$ differs in all their $n$ entries. Since $\A$ is $\epsilon$-differentially private, for all $i
\in[K]$, we have $Pr[A(\D^{(1)})\in B(\D^{(i)})]\ge \frac{1}{2}e^{-\epsilon n}$. Since all $B(\D^{(i)})$ are mutually disjoint, then
\begin{equation}
    \frac{K}{2}e^{-\epsilon n}\le \sum_{i=1}^K Pr[\A(\D^{(1)})\in B(\D^{(i)})]\le 1
\end{equation}

which implies that $n>n^{\star}$ for sufficiently large $p$, contradicting the fact that $n\le n^{\star}$. Hence, there must exist a dataset $\D^{(i)}$ on which $A$ makes an $\ell_2$-error on estimating $q(\D)$ which is at least $1/16$ with probability at least $1/2$. Note also that the $\ell_2$ norm of the sum of the entries of such $\D^{(i)}$ is $n$.

Next, we consider the case where $n>n^{\star}$. As before, we construct $K=2^{\frac{p}{2}}$ datasets $\tilde{\D}^{(1)},\cdots,\tilde{\D}^{(K)}$ of size $n$ where for every $i\in [K]$, the first $n^{\star}$ elements of each dataset $\tilde{\D}^{(i)}$ are the same as dataset $\D^{(i)}$ from before whereas the remaining $n-n^{\star}$ elements are $\bzero$. 

Note that any two distinct datasets $\tilde{\D}^{(i)},\tilde{\D}^{(j)}$ in this collection differ in exactly $n^{\star}$ entries. Let $\A$ be any $\epsilon$-differentially private algorithm for answering $q$. Suppose that for every $i\in[K]$, with probability at least $1/2$, we have that
\begin{equation}
    \|\A(\tilde{\D}^{(i)})-q(\tilde{\D}^{(i)})\|_2< \frac{n^{\star}}{16n}
\end{equation}

Note that for all $i\in[K]$, we have that $q(\tilde{\D}^{(i)})=\frac{n^*}{n} q(\D^{(i)})$. Now, we define an algorithm $\tilde{\A}$ for answering $q$ on datasets $\D$ of size $n^{\star}$ as follows. First, $\tilde{\A}$ appends $\bzero$ as above to get a dataset $\tilde{\D}$ of size $n$. Then, it runs $\A$ on $\tilde{\D}$ and outputs $\frac{n^* \A(\tilde{\D})}{n}$. Hence, by the post-processing propertry of differential privacy, $\tilde{\A}$ is $\epsilon$-differentially private since $\A$ is $\epsilon$-differentially private. Thus for every $i\in[K]$, with probability at least $1/2$, we have that $||\tilde{\A}(\D^{(i)})-q(\D^{(i)})||_2< \frac{1}{16}$. However, this contradicts our result in the first part of the proof. Therefore, there must exist a dataset $\tilde{\D}^{(i)}$ in the above collection such that, with a probability at least $1/2$,
\begin{equation}
    \|\A(\tilde{\D}^{(i)})-q(\tilde{\D}^{(i)})\|_2\ge \frac{n^{\star}}{16n}\ge \frac{d}{320\epsilon n}
\end{equation}
Note that the $\ell_2$ norm of the sum of entries of such $\tilde{D}^{(i)}$ is always $n^{\star}$.
\end{proof}

\subsection{Proof of Theorem \ref{thm1}}

\reductiontwo*
\begin{proof}
We can prove this theorem directly by combining the lower bound in \cite{bst14} and our reduction approach (Theorem~\ref{thm4}), but we try to give a complete proof as an example to demonstrate how does our black-box reduction approach work out.

Let $\A$ be an $\epsilon$-differentially private algorithm for minimizing $L$ and let $\theta^{priv}$ denote its output, define $r:=\theta^{priv}-\theta^*$. First, observe that for any $\theta\in \R^d$ and dataset $\D$ as constructed in Lemma \ref{lem5} (recall that $\D$ consists of $n^*$ copies of a vector $z\in \{\frac{1}{\sqrt{d}},-\frac{1}{\sqrt{d}}\}^d$ and $n-n^*$ copies of $\mathbf{0}$).
\begin{equation}
    L(\theta^*;\D)=\frac{n-n^*}{n}\max\{0,\|\theta^*\|_2-1\}+\frac{n^*}{n} \min_{\|y\|_2\le 1} (-\langle y,z \rangle+\|\theta^*-y\|_2)= -\frac{n^*}{n}
\end{equation}
when $\theta^*=z$, and also
\begin{align*}
    L(\theta^{priv};\D)&=\frac{n-n^*}{n}\max\{0,\|\theta^{priv}\|_2-1\}+\frac{n^*}{n} \min_{\|y\|_2\le 1} (-\langle y,z \rangle+\|\theta^{priv}-y\|_2)\\
    &\ge \frac{n^*}{n} \min_{\|y\|_2\le 1} (-\langle y,z \rangle+\|\theta^{priv}-y\|_2)\\
    &=\frac{n^*}{n} \min_{\|y\|_2\le 1} (-\langle y,z \rangle+\|r+z-y\|_2)\\
    &\text{(because $\theta^*=z$)}\\
    &\ge \frac{n^*\min\{1,\|r\|^2_2\}}{8n} -\frac{n^*}{n}
\end{align*}
the last inequality follows by discussing the norm of $y-z$. If $\|y-z\|_2\le \|r\|_2/2$, then 
\begin{equation}
    \|r+z-y\|_2\ge \|r\|_2/2\ge \min\{1,\|r\|^2_2\}/2
\end{equation}
combining with the fact that $|\langle y,z \rangle|\le 1$ proves the last inequality. 

If $\|y-z\|_2\ge \|r\|_2/2$, then we have $\min_{\|y\|_2\le 1} -\langle y,z \rangle\ge -1+\frac{\|r\|_2^2}{8}$. To prove this, we assume $z=e_1$ without loss of generality and $y-z=(x_1,...,x_d)$ where $\sum_{i=1}^d x_i^2\ge \|r\|_2^2/4$. Since $\|y\|_2=\|y-z+z\|_2\le 1$, we must have 
\begin{equation}
     1+\sum_{i=1}^d x_i^2+2x_1\le 1
\end{equation}
Thus $-\langle y,z \rangle= -1-\langle y-z,z \rangle=-1-x_1\ge -1+\|r\|_2^2/8$ as desired, which finishes the discussion on the second case.

From the above result we have that 
\begin{equation}
    L(\theta^{priv};\D)-L(\theta^*;\D)\ge \frac{n^*\min\{1,\|r\|^2_2\}}{8n} 
\end{equation}
To proceed, suppose for the sake of a contradiction, that for every dataset $\D=\{z_1,...,z_n\}\subset \{\frac{1}{\sqrt{d}},-\frac{1}{\sqrt{d}}\}^d\cup \{ \mathbf{0}\}$ with $\|\sum_{i=1}^n z_i\|_2=n^*$, with probability more
than $1/2$, we have that $\|\theta^{priv}-\theta^*\|_2=\|r\|_2\ne \Omega(1)$. Let $\tilde{\A}$ be an $\epsilon$-differentially private algorithm that first runs $\A$ on the data and then outputs $\frac{n^*}{n} \theta^{priv}$. Recall that $q(\D)=\frac{n^*}{n}\theta^*$, this implies that for every dataset $\D=\{z_1,...,z_n\}\subset \{\frac{1}{\sqrt{d}},-\frac{1}{\sqrt{d}}\}^d\cup \{ \bzero\}$ with $\|\sum_{i=1}^n z_i\|_2=n^*$, with probability more than $1/2$, $\|\tilde{\A}(\D)-q(\D)\|_2\ne \Omega(\min(1,\frac{d}{n\epsilon}))$ which contradicts Lemma \ref{lem5}. Thus, there must exists a dataset $\D=\{z_1,...,z_n\}\subset \{\frac{1}{\sqrt{d}},-\frac{1}{\sqrt{d}}\}^d\cup \{ \bzero\}$ with $\|\sum_{i=1}^n z_i\|_2=n^*$, such that with pr
obability more than $1/2$, we have $\|r\|_2=\|\theta^{priv}-\theta^*\|_2= \Omega(1)$, and as a result
\begin{equation}
    L(\theta^{priv};\D)-L(\theta^*;\D)=\Omega(\min(1,\frac{d}{n\epsilon}))
\end{equation}

\end{proof}
\section{Omitted proof for Section~\ref{sec:improvebound}}

\subsection{Fingerprinting codes}
Fingerprinting code was first introduced in \cite{boneh1998collusion}, developed and frequently used to demonstrate lower bounds in the DP community \cite{buv18,su15,steinke2015interactive}. To overcome the challenge discussed before, we slightly modify the definition of the fingerprinting code used in this work.

\begin{definition}[$\ell_1$-loss Fingerprinting Code]
\label{def:FC}
A $\gamma$-complete, $\gamma$-sound, $\alpha$-robust $\ell_1$-loss fingerprinting code for $n$ users with length $d$ is a pair of random variables $\D\in \{0, 1\}^{n\times d}$ and $\Trace$ : $[0,1]^d \to 2^{[n]}$ such that the following hold:
\paragraph{Completeness:}
For any fixed $\M: \{0, 1\}^{n\times d}\to [0,1]^d$, 
\begin{align*}
    \Pr \Big[ L(\M(\D);\D)-\min_{\theta}L(\theta;\D)\le \alpha d \\~~ \land (\Trace (\M(\D))=\emptyset)  \Big]\le \gamma.
\end{align*}
\paragraph{Soundness:}
For any $i\in [n]$ and fixed $M: \{0, 1\}^{n\times d}\to [0,1]^d$,
$$ 
\Pr [i\in \Trace (M(\D_{-i}))]\le \gamma,
$$
where $\D_{-i}$ denotes $\D$ with the $i$th row replaced by some fixed element of $\{0, 1\}^d$.
\end{definition}

Definition~\ref{def:FC} is similar to the one in \cite{su15} (See Definition 3.2 in \cite{su15}), except that their requirement of completeness is $\Pr[||\M(\D)-q(\D)\|_1\le \alpha d\land \Trace (\M(\D))=\emptyset ]\le \gamma$.
As discussed before, they use the fingerprinting code in their version to build a lower bound on the mean estimation, while we modify the definition and build a lower bound on the DP-ERM under our set-up.

Following the optimal fingerprinting construction \cite{tar08}, and subsequent works \cite{buv18} \cite{bst14}, we have the following result demonstrating the existence of fingerprinting code in our version. 

\begin{restatable}{lemma}{appone}
\label{lm:fingerprinting_code}
For every 
$n\ge1$, and $\gamma \in (0, 1]$, there exists a $\gamma$-complete, $\gamma$-sound, $1/150$-robust $\ell_1$-loss fingerprinting code for $n$ users with length $d$ where 
\begin{align*}
    d=O(n^2\log(1/\gamma).
\end{align*}
\end{restatable}

\subsection{Proof of Lemma~\ref{lm:fingerprinting_code}}

\begin{proof}
We want to find $\alpha$ such that any set satisfying the completeness condition in the above definition is a subset of the $F_{\beta}$ set of \cite{buv18} after rounded to binary numbers, which is
$$
F_{\beta}(\D)=\left\{c'\in \{0,1\}^d | \Pr_{j\in [d]} [\exists i\in [n], c'_j=\D_{ij}]\ge 1-\beta \right\}
$$
Suppose, round the output $\M(\D)\in[0,1]^d$ to a binary vector $c\in\{0,1\}^d$ where $c\notin F_{\beta}(\D)$, 
then it makes an "illegal" bit on at least $\beta d$ columns, where each of these columns shares the same number (all-one or all-minus-one columns). It means that on each of these columns, $\M(\D)$ has the opposite sign to the shared number, which means on this column, say $i$, the induced loss is lower bounded:
$$
\frac{1}{n} \sum_{j=1}^n (|(\M(\D)_i-\D_{ij}|-|\sign (\bar{\D_i})-\D_{ij}|) =\frac{1}{n} \sum_{j=1}^n |(\M(\D)_i-\D_{ij}|\ge 1,
$$
which means $L(\M(\D);\D)-\min_{\theta}L(\theta;\D)\ge \beta d/2$.
By Theorem~\ref{thm:construction_fc} we get $\beta=1/75$ and conclude our proof.
\end{proof}

\subsection{Proof of Lemma \ref{lemep}}
\apptwo*

\begin{proof}
The proof uses a black-box reduction, therefore doesn't depend on $Q$. The direction that $O(n^*/\epsilon)$ samples are sufficient is equal to proving the assertion that given a $(1,o(1/n))$-differentially private algorithm $\A$, we can get a new algorithm $\A'$ with $(\epsilon,o(1/n))$-differential privacy at the cost of shrinking the size of the dataset by a factor of $\epsilon$. 

Given input $\epsilon$ and a dataset $X$, we construct $A'$ to first generate a new dataset $T$ by selecting each element of $X$ with probability $\epsilon$ independently, then feed $T$ to $\A$. Fix an event $S$ and two neighboring datasets $X_1,X_2$ that differs by a single element $i$. Consider running $\A$ on $X_1$. If $i$ is not included in the sample $T$, then the output is distributed the same as a run on $X_2$. On the other hand, if $i$ is included in the sample $T$, then the behavior of $\A$ on $T$ is only a factor of $e$ off from the behavior of $\A$ on $T\setminus \{i\}$. Again, because of independence, the distribution of $T\setminus \{i\}$ is the same as the distribution of $T$ conditioned on the omission of $i$.

For a set $X$, let $p_{X}$ denote the distribution of $\A(X)$, we have that for any event $S$,
\begin{align*}
    & p_{X_1}(S)=(1-\epsilon) p_{X_1}(S | i \notin T) +\epsilon p_{X_1}(S | i \in T)\\
    &\le (1-\epsilon) p_{X_2}(S) +\epsilon( e\cdot p_{X_2}(S)+\delta)\\
    &\le \exp(2\epsilon) p_{X_2}(S)+\epsilon \delta
\end{align*}
A lower bound of $p_{X_1}(S)\ge  \exp(-\epsilon) p_{X_2}(S)-\epsilon \delta/e$ can be obtained similarly. To conclude, since $\epsilon \delta=o(1/n)$ as the sample size $n$ decreases by a factor of $\epsilon$, $\A'$ has $(2\epsilon,o(1/n))$-differential privacy. The size of $X$ is roughly $1/\epsilon$ times larger than $T$, combined with the fact that $\A$ has sample complexity $n^*$ and $T$ is fed to $\A$, $\A'$ has sample complexity at least $\Theta(n^*/\epsilon)$.

For the other direction, simply using the composability of differential privacy yields the desired result. In particular, by the $k$-fold adaptive composition theorem in \cite{dmns06}, we can combine $1/\epsilon$ independent copies of $(\epsilon,\delta)$-differentially private algorithms to get an $(1, \delta/\epsilon)$ one and notice that if $\delta=o(1/n)$, then $\delta/\epsilon=o(1/n)$ as well because the sample size $n$ is scaled by a factor of $\epsilon$ at the same time, offsetting the increase in $\delta$.
\end{proof}

\subsection{Proof of Lemma~\ref{lm:error_by_appending}}
\begin{proof}
Without loss of generality, we can assume $z_{k(i-1)+1}'=z_{k(i-1)+2}'=\cdots=z_{ki}'=z_i$, and $z_{n-kn_k+1}'=z_{kn_k+2}'=\cdots=z_{n}'=0$.
With this observation, we know 
\begin{align*}
     &~~|\sum_{i=1}^{n_k}|q-z_i|/n_k-\sum_{i=1}^{n}|q-z_i'|/n|\\
     &=~|\sum_{i=1}^{n_k}|q-z_i|(1/n_k-k/n) -\sum_{i=n-kn_k+1}^{n}q/n|\\
     &\le~|\sum_{i=1}^{n_k}|q-z_i|(1/n_k-k/n)|+|\sum_{i=n-kn_k+1}^{n}q/n|\\
     &\le~n_k(\frac{1}{k/n-1}-\frac{k}{n})+k/n\le 3k/n.
\end{align*}
\end{proof}



\subsection{Proof of Theorem~\ref{thm2}}

\appthree*
\begin{proof}

Let $k=\Theta(\log(1/\delta))$ be a parameter to be determined later satisfying $k/n<1/6000$, and $n_k=\lfloor n/k\rfloor$.
Consider the case when $d\geq d_{n_k}$ first, where $d_{n_k}=O(\epsilon^2 n_k^2\log(1/\delta))$.

Without loss of generality, we assume $\epsilon=1$ due to Lemma \ref{lemep}, and $d_{n_k}=O(n_k^2\log(1/\delta))$ corresponds to the number in Lemma~\ref{lm:fingerprinting_code} where we set $\gamma=\delta$.

We use contradiction to prove that for any $(\epsilon,\delta)$-DP mechanism $\M$, there exists some $\D\in\{0,1\}^{n\times d}$ such that
\begin{equation}
    \E[L(\M(\D);\D)-L(\theta^\star;\D)]\geq \Omega(d).
\end{equation}

Assume for contradiction that $\M:\{0,1\}^{n\times d}\rightarrow [0,1]^{ d}$ is a (randomized) $(\epsilon,\delta)$-DP mechanism such that
\begin{align*}
    \E[L(\M(\D);\D)-L(\theta^\star;\D)]< \frac{d}{3000}
\end{align*}
for all $\D\in\{0,1\}^{n\times d}$.
We then construct a mechanism $\M_k=\{0,1\}^{n_k\times d}$ with respect to $\M$ as follows: with input $\D^k\in\{0,1\}^{n_k\times d}$, $\M_k$ will copy $\D^k$ for $k$ times and append enough 0's to get a dataset $\D\in\{0,1\}^{n\times d}$. 
The output is $\M_k(\D^k)=\M(\D)$. $\M_k$ is $(k,\frac{e^{k}-1}{e^-1}\delta)$-DP by the group privacy.

We consider algorithm $\A_{FP}$ to be the adversarial algorithm in the fingerprinting codes, which rounds the output $\M_k(\D^k)$ to the binary vector, i.e., rounding those coordinates with values no less than 1/2 to 1 and the remaining 0, and let $c=\A_{FP}(\M(\D))$ be the vector after rounding.
As $\M_k$ is $(k,\frac{e^k-1}{e-1}\delta)$-DP, $\A_{FP}$ is also $(k,\frac{e^k-1}{e-1}\delta)$-DP.


Considering the $\ell_1$ loss, we can account for the loss caused by each coordinate separately. 
Recall that $\M_k(\D^k)=\M(\D)$.
Thus we have that
\begin{align*}
    &\E[L(\M_k(\D^k);\D^k)-L(\theta^\star;\D^k)]\\
    = & ~\E[L(\M(\D);\D^k)-L(\theta^\star;\D^k)]\\
    =& ~\E[L(\M(\D);\D^k)]-\E[L(\M(\D);\D)]+L(\theta^\star;\D)-L(\theta^\star;\D^k)+\E[L(\M(\D);\D)-L(\theta^\star;\D)]\\
    \le& ~ 6kd/n  +d/3000\\
    \le&~ d/900,
\end{align*}
where we use Lemma~\ref{lm:error_by_appending} for the third line. 

By Markov Inequality, we know that 
\begin{align*}
    \Pr[L(\M_k(\D^k);\D^k)-L(\theta^\star;\D^k)]>\frac{d}{150}]\leq 1/5.
\end{align*}
Lemma \ref{lm:fingerprinting_code} implies
\begin{align*}
    ~&~\Pr[L(\M_k(\D^k);\D^k)-L(\theta^\star;\D^k)\leq d/150 \bigwedge \Trace(\D^k,c)=\perp]
    \leq \delta.
\end{align*}

By union bound, we can upper bound the probability
$\Pr[\Trace(\D^k,c)=\perp]\leq 1/5+ \delta\leq 1/2$.
As a result, there exists $i^*\in [n_k]$ such that 
\begin{align}
    \Pr[i^*\in \Trace(\D^k,c)]\geq 1/(2n_k).
\end{align}

Consider the database with $i^*$ removed, denoted by $\D^k_{-i^*}$. Let $c'=\A_{FP}(\M(\D^k_{-i^*}))$ denote the vector after rounding.
By the second property of fingerprinting codes, we have that
\begin{align*}
    \Pr[i^*\in \Trace(\D^k_{-i^*},c')]\leq \delta.
\end{align*}
By the differential privacy and post-processing property of $\M$, 
\begin{align*}
    \Pr[i^*\in\Trace(\D^k,c)]\leq e^{k}\Pr[i^*\in \Trace(\D^k_{-i^*},c')]+\frac{e^{k}-1}{e-1}\delta.
\end{align*}
which implies that
\begin{align}
\label{eq:contradic}
    \frac{1}{2n_k}\leq e^{k+1}\delta.
\end{align}
Recall that $2^{-O(n)}<\delta<o(1/n)$, and Equation~(\ref{eq:contradic}) suggests $k/n \leq 2e^k/\delta$ for all valid $k$. But it is easy to see there exists $k=\Theta(\log(1/\delta))$ and $k<n/6000$ to make this inequality false, which is contraction.
As a result, there exists some $\D\in\{0,1\}^{n\times d}$ such that
\begin{align*}
    \E[L(\M(\D);\D)-L(\theta^\star;\D)]\geq \frac{d}{3000}=\Omega(d).
\end{align*}
For the $(\epsilon,\delta)$-DP case when $\epsilon<1$, setting $Q$ to be the condition 
\begin{align*}
    \E[L(\M(\D);\D)-L(\theta^\star;\D)]=O(d)
\end{align*}
for all $\D\in\{0,1\}^d$
in Lemma \ref{lemep}, we have that any $(\epsilon,\delta)$-DP mechanism $\M$ which satisfies $Q$ for all $\D\in\{0,1\}^{n\times p}$ must have $n\geq\Omega(\sqrt{d\log(1/\delta)}/\epsilon)$.
In another word, for $d\ge O(\epsilon^2n^2/\log(1/\delta))$, for any $(\epsilon,\delta)$-DP mechanism $\M$, there exists some $\D\in\{0,1\}^d$ such that 
\begin{align*}
    \E[L(\M(\D);\D)-L(\theta^\star;\D)]\geq \Omega(d).
\end{align*}

Now we consider the case when $d<d_{n_k}$, i.e., when $n>n^\star\triangleq \Omega(\sqrt{d\log(1/\delta)}/\epsilon)$.
Given any dataset $\D\in\{0,1\}^{n^\star\times d}$, we construct a new dataset $\D'$ based on $\D$ by appending dummy points to $\D$: 
Specifically, if $n-n^\star$ is even, we append $n-n^{\star}$ rows among which half are 0 and half are $\{1\}^d$. 
If $n-n^{\star}$ is odd, we append $\frac{n-n^{\star}-1}{2}$ points 0, $\frac{n-n^{\star}-1}{2}$ points $\{1\}^d$ and one point $\{1/2\}^{d}$. 

Denote the new dataset after appending by $\D'$, we will draw contradiction if there is an $(\epsilon,\delta)$-DP algorithm $\M'$ such that $\E[L(\M(\D');\D')-L(\theta^{\star};\D')]=o(n^{\star} d/n)$ for all $\D'$, by reducing $\M'$ to an $(\epsilon,\delta)$-DP algorithm $\M$ which satisfies $\E[L(\M(\D);\D)-L(\theta^{\star};\D)]=o(d)$ for all $\D$.

We construct $\M$ by first constructing $\D'$, and then use $\M'$ as a black box to get $\M(\D)=\M'(\D')$.
It's clear that such algorithm for $\D$ preserves  $(\epsilon,\delta)$-differential privacy. 
It suffices to show that if
\begin{equation}
    \E[L(\M'(\D');\D')-L(\theta^{\star};\D')]=o(n^{\star} d/n),
\end{equation}
then $L(\M(\D);\D)-L(\theta^{\star};\D)=o(d)$, which contradicts the previous conclusion for the case $n\leq n^\star$. Specifically, if $n-n^\star$ is even, we have that
\begin{align*}
    n^\star\E[L(\M(\D);\D)-L(\theta^\star;\D)]=n\E[L(\M'(\D');\D')-L(\theta^\star;\D')].
\end{align*}
and if $n-n^\star$ is odd, we have that 
\begin{align*}
    n^\star\E[L(\M(\D);\D)-L(\theta^\star;\D)]\leq n\E[L(\M'(\D');\D')-L(\theta^\star;\D')]+d/2,
\end{align*}
both leading to the desired reduction.
We try to explain the above two cases in more detail. 
If $n-n^*$ is even, then the minimizer of $L(; \D)$ and $L(\theta^*;\D)$ are the same. And the distributions of the $\M(\D)$ and $\M'(\D')$ are identical and indistinguishable. 
Multiplying $n^*$ or $n$ depends on the number of rows (recall that we normalize the objective function in ERM). The second inequality is because we append one point $\{1/2\}^d$, which can only increase the loss $(\|{1/2}^d-\theta^*\|_1)$ by $d/2$ in the worst case.

Combining results for both cases, we have the following:
\begin{equation}
    \E[L(\theta^{priv};\D)-L(\theta^{\star};\D)]=\Omega(\min(d,\frac{d n^* }{n}))=\Omega(\min(d,\frac{d\sqrt{d\log(1/\delta)}}{n\epsilon})).
\end{equation}

Setting Lipschitz constant $G=\sqrt{d}$ and diameter $C=\sqrt{d}$ completes the proof. 
\end{proof}

\subsection{Proof of Theorem~\ref{thm3}}
\appfour*
\begin{proof}
We use the same construction as in Theorem \ref{thm2} which considers $\ell_2$ geometry. We only need to calculate the Lipschitz constant $G$ and the diameter of the domain $\K$.

For the Lipschitz constant $G$, notice that our loss is the $\ell_1$ norm: $\ell(\theta;z)=\|\theta-z\|_1$. It is evident  that it is $(d^{1-\frac{1}{p}})$-Lipschitz w.r.t. $\ell_p$ geometry.

For the domain, i.e., the unit $\ell_{\infty}$ ball $\K$, it obvious that its diameter w.r.t. $\ell_p$ geometry is $C=d^{\frac{1}{p}}$.
To conclude, we find that for any $\ell_p$ geometry where $p\geq1$, we have $GC=d$ which is independent of $p$. 
The bound holds for any $\ell_p$ geometry by applying Theorem~\ref{thm2}.
\end{proof}

\end{document}